\documentclass[dvipsnames]{article} %
\usepackage{colm2024_conference}

\usepackage{booktabs}
\usepackage{enumitem}
\usepackage{wrapfig}
\usepackage{graphicx}
\usepackage[misc]{ifsym}
\usepackage{multicol} 
\usepackage{microtype}
\usepackage{amsmath}
\usepackage{colortbl}
\usepackage[utf8]{inputenc}
\usepackage[T1]{fontenc}
\definecolor{lightgray}{rgb}{0.9,0.9,0.9}
\usepackage{caption}
\usepackage{subcaption}
\usepackage{graphicx}
\usepackage{subcaption}
\usepackage{setspace}
\usepackage{url}
\usepackage{multirow}
\usepackage{tabularx}
\usepackage{blindtext}
\usepackage{pgfplots}
\pgfplotsset{compat=1.18} 
\usepackage{tikz}
\usetikzlibrary{er,positioning,bayesnet}
\usepackage{makecell}
\usepackage{tipa}
\usepackage{siunitx}
\usepackage{nicefrac}
\usepackage{listings}
\usepackage[raster,skins, most]{tcolorbox} %
\usepackage{xltabular}
\usepackage{adjustbox}
\usepackage{xurl}
\usepackage{rotating}
\usepackage[normalem]{ulem}

\usepackage{fontawesome}

\usepackage[table]{xcolor}

\useunder{\uline}{\ul}{}

\newcommand*\justify{%
  \fontdimen2\font=0.4em
  \fontdimen3\font=0.2em
  \fontdimen4\font=0.1em
  \fontdimen7\font=0.1em
  \hyphenchar\font=`\-
}

\renewcommand{\texttt}[1]{%
  \begingroup
  \ttfamily
  \begingroup\lccode`~=`/\lowercase{\endgroup\def~}{/\discretionary{}{}{}}%
  \begingroup\lccode`~=`[\lowercase{\endgroup\def~}{[\discretionary{}{}{}}%
  \begingroup\lccode`~=`.\lowercase{\endgroup\def~}{.\discretionary{}{}{}}%
  \catcode`/=\active\catcode`[=\active\catcode`.=\active
  \justify\scantokens{#1\noexpand}%
  \endgroup
}

\usepackage{makecell}
\usetikzlibrary{tikzmark}
\makeatletter
\newcommand*\myfontsize{%
  \@setfontsize\myfontsize{7}{8}%
}
\makeatother

\definecolor{uclablue}{RGB}{159, 195, 224}

\definecolor{uclagold}{RGB}{255, 240, 180}

\definecolor{aliceblue}{RGB}{255, 238, 241}

\definecolor{cadmiumgreen}{rgb}{0.0, 0.42, 0.24}

\definecolor{myred}{rgb}{0.7, 0.3, 0.0}
\definecolor{myblue}{rgb}{0.2, 0.3, 0.6}
\definecolor{babygreen}{rgb}{0.85, 0.97, 0.85}

\definecolor{purple1}{RGB}{126, 107, 196}
\definecolor{purple2}{RGB}{199, 158, 207}
\definecolor{purple3}{RGB}{214, 200, 255}
\definecolor{purple4}{RGB}{254, 240, 255}

\definecolor{deepblue}{RGB}{48, 58, 82}

\newcommand{\symboletongyi}{\raisebox{0pt}{~\includegraphics[scale=0.012]{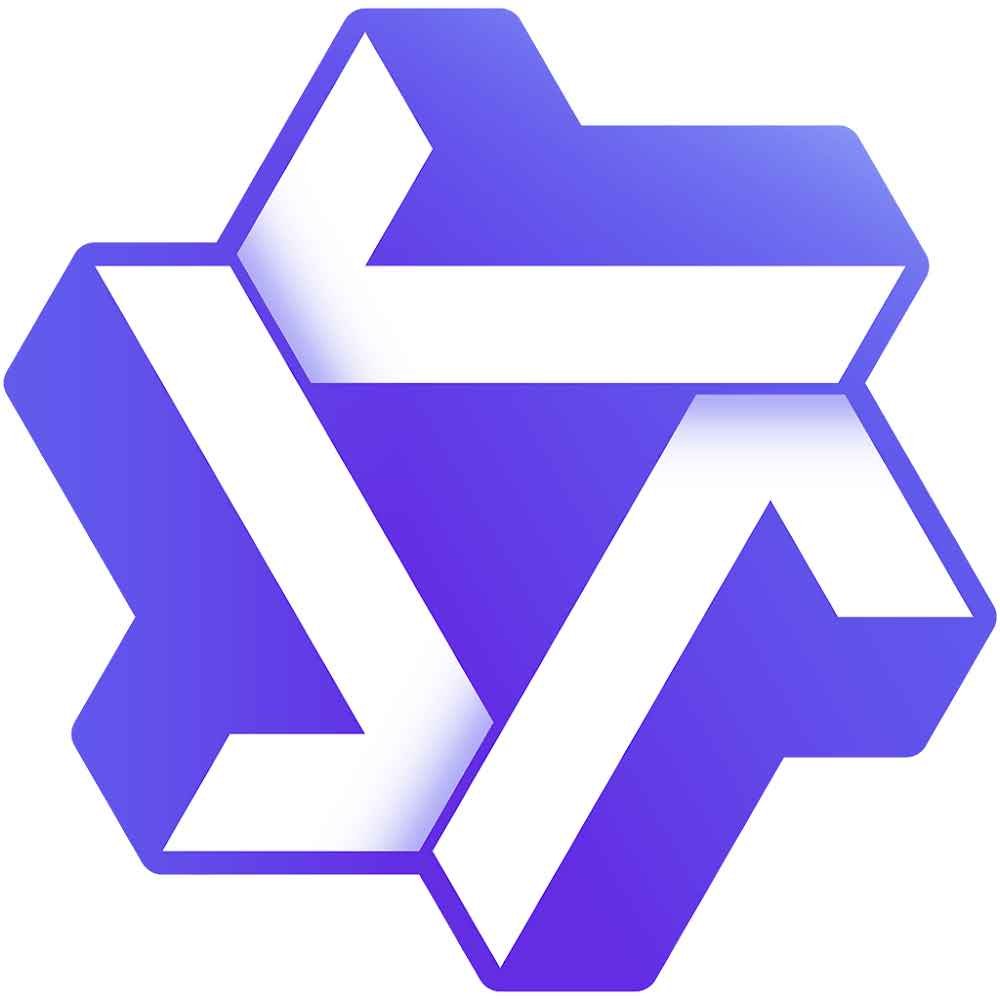}}~}

\definecolor{deepPurple}{HTML}{330066}

\definecolor{uclablue_old}{rgb}{0.15, 0.45, 0.68}
\hypersetup{
    breaklinks,
    citecolor=uclablue_old,
    colorlinks=true,
}

\newtcolorbox{mybox}[2][]
  {colback = black!5!white, colframe = black!75!black, fonttitle = \bfseries,
    colbacktitle = black!100!black, enhanced, before upper={\fontsize{8}{11}\obeyspaces\obeylines\selectfont}, fontupper=\selectfont,
    attach boxed title to top left={yshift=-2.2mm,xshift=4mm},
    title=#2,#1}

\usepackage{graphicx}

\usepackage{amsmath}
\usepackage{amssymb}
\usepackage{mathtools}
\usepackage{amsthm}

\usepackage[most]{tcolorbox}  

\usepackage{booktabs}
\usepackage{caption}
\usepackage{makecell}

\theoremstyle{plain}
\newtheorem{theorem}{Theorem}[section]

\newtheorem{lemma}[theorem]{Lemma}

\theoremstyle{definition}

\theoremstyle{remark}

\newtheorem*{remark*}{Remark}

\usepackage{multirow}
\usepackage[dvipsnames]{xcolor}
\usepackage[table]{xcolor}
\usepackage[normalem]{ulem}
\usepackage{pifont}
\useunder{\uline}{\ul}{}

\usepackage{wrapfig}
\usepackage{lipsum}
\usepackage{array}
\usepackage{tabulary}
\usepackage{tabularx}
\usepackage{listings}
\usepackage{enumitem}
\usepackage{algorithm}
\usepackage{algorithmic}

\lstdefinestyle{mystyle}{
    basicstyle=\ttfamily\footnotesize,
    breaklines=true, 
    columns=flexible, 
    captionpos=b,
}
\usepackage{subcaption}

\newcommand{\ie}{\emph{i.e., }}
\newcommand{\eg}{\emph{e.g., }}

\newcommand{\st}{\emph{s.t. }}

\newcommand{\cf}{\emph{cf. }}

\newcommand{\updated}[1]{#1}  

\title{%
\begin{tabular}[t]{l} 
  \parbox[t]{0.8\textwidth}{\centering 
    On the Direction of RLVR Updates for LLM Reasoning: Identification and Exploitation
  }
\end{tabular}
}

\author{%
\large \symboletongyi Qwen Pilot Team, Alibaba Group\thanks{Full author list available in the \hyperref[sec:contribution]{Contributions} section.}%
  \\[1em] 
}

\begin{document}
\maketitle

\begin{abstract}
Reinforcement learning with verifiable rewards (RLVR) has substantially improved the reasoning capabilities of large language models. 
While existing analyses identify that RLVR-induced changes are sparse, they primarily focus on the \textbf{magnitude} of these updates, largely overlooking their \textbf{direction}. 
In this work, we argue that the direction of updates is a more critical lens for understanding RLVR's effects, which can be captured by the signed, token-level log probability difference $\Delta\log p$ between the base and final RLVR models.
Through statistical analysis and token-replacement interventions, we demonstrate that $\Delta\log p$ more effectively identifies sparse, yet reasoning-critical updates than magnitude-based metrics (\eg divergence or entropy).
Building on this insight, we propose two practical applications:
(1) a \textit{test-time extrapolation} method that amplifies the policy along the learned $\Delta\log p$ direction to improve reasoning accuracy without further training;
(2) a \textit{training-time reweighting} method that focuses learning on low-probability (corresponding to higher $\Delta\log p$) tokens, which improves reasoning performance across models and benchmarks.
Our work establishes the direction of change as a key principle for analyzing and improving RLVR.

\end{abstract}

\section{Introduction}

Recent advances have substantially improved the reasoning capabilities of large language models, giving rise to powerful reasoning-centric models such as OpenAI o1 \citep{Openai-O1}, Deepseek R1 \citep{Deepseek-R1}, Gemini 2.5 \citep{Gemini-2.5}, and Qwen3 \citep{Qwen3}.
A key algorithmic driver of this progress is reinforcement learning with verifiable rewards (RLVR) \citep{Deepseek-R1, kimi-k1.5, Qwen3}, 
which fine-tunes a model's generation policy using feedback from task-specific verifiers,
thereby eliciting and amplifying the reasoning ability.

To elucidate how RLVR confers its gains, a natural lens is to compare what changes in the final RL-trained model $\pi_\mathrm{RL}$ relative to its base counterpart $\pi_\mathrm{Base}$ \citep{iclr25_dynamics}.
Previous analyses have consistently shown that the RLVR-induced changes are sparse, impacting only a small subset of tokens in the output sequence. 
For example, \citet{entropy_28_rule} associate these changes with high-entropy tokens, \citet{UniReason} corroborate the sparsity by measuring the KL divergence between $\pi_\mathrm{Base}$ and $\pi_\mathrm{RL}$, 
while \citet{over_dominate} and \citet{decomposing_zhaoxin} attribute this sparsity to selective gradient updates during RLVR training.
\updated{However, when studying the difference between base and RLVR models, prior studies primarily emphasize the \textbf{magnitude of change}, but largely overlook the \textbf{direction} in their distributions.}
As shown in Fig.~\ref{fig:teaser}\textcolor{red}{\hyperref[fig:teaser]{(b)}}, 
magnitude-based metrics (\eg entropy, KL divergence) yield nearly identical histograms for the base and final RLVR models,
indicating that magnitude alone is insufficient to characterize the transformation from $\pi_\mathrm{Base}$ to $\pi_\mathrm{RL}$.

To address this gap, we directly quantify directional shifts in the model’s distribution using the signed, token-level log-probability difference:
\begin{equation}\label{eqn:logp_diff}
    \Delta\log p(y_t|x,y_{<t}) = \log\pi_\mathrm{RL}(y_t|x,y_{<t}) - \log \pi_\mathrm{Base}(y_t|x,y_{<t}),
\end{equation}
which captures how RLVR shifts the probability mass on each token, with positive values indicating increased probabilities and negative values vice versa.
As shown in Fig.~\ref{fig:teaser}\textcolor{red}{\hyperref[fig:teaser]{(b)}}, histograms of $\Delta\log p$ exhibit a clear bimodal pattern with two distinct tails, highlighting a clear directional signature absent in magnitude-based metrics.
This metric can reveal which token RLVR prioritizes, such as reasoning-critical tokens (\eg those enhancing reasoning correctness) versus irrelevant ones.
We further validate its utility via a token replacement intervention \citep{crosssample}: for each metric, we identify salient positions and replace the base model's tokens with the RLVR model’s choices at those positions during generation \updated{(\cf Algo.~\ref{alg:replace})}.
As shown in Fig.~\ref{fig:teaser}\textcolor{red}{\hyperref[fig:teaser]{(c)}}, selecting by $\Delta\log p$ reaches RLVR-level performance with the fewest substitutions, pinpointing tokens where RLVR learns reasoning-critical updates. 
These findings underscore a key principle: \textit{analyzing the direction of changes, rather than solely their magnitude, provides deeper insights.}
The signed log-probability difference provides a practical and effective handle for this diagnostic analysis.

\begin{figure}[t]
    \centering
    \includegraphics[width=0.95\linewidth]{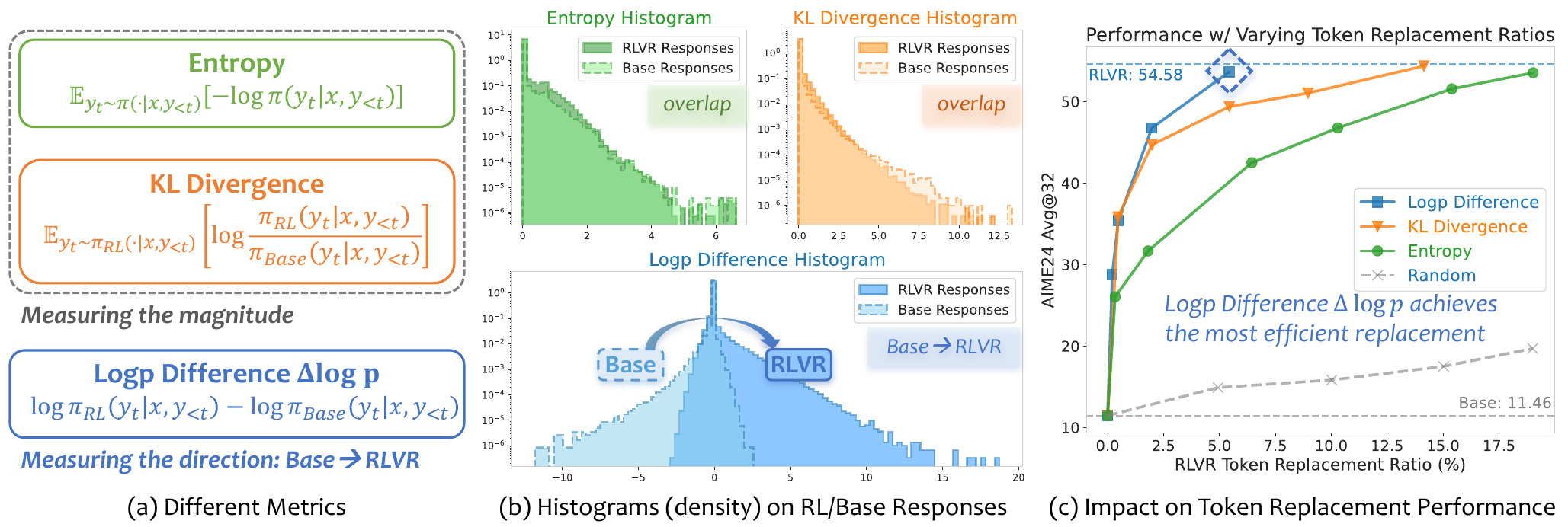}
    \caption{(a) Token-level metrics for analyzing RLVR updates. (b) Histograms of each metric on responses generated by base and RLVR models. With a log-scale y-axis, most values concentrate near zero for all metrics, but only $\Delta \log p$ shows a directional shift distinguishing RLVR from the base model. (c) Token‑replacement performance: replacing base tokens with RLVR choices at positions selected by each metric, where $\Delta \log p$ recovers RLVR performance with the fewest replacements.}
    \label{fig:teaser}
    \vspace{-8pt}
\end{figure}

Building on this principle, we first propose a test-time augmentation that extrapolates the RLVR policy’s distribution along the $\Delta\log p$ direction for reasoning-critical tokens selectively, amplifying reasoning-related updates and improving accuracy without additional training. 
Furthermore, we observe that tokens with the largest $\Delta\log p$ consistently correspond to low-probability tokens during RLVR training.
Motivated by this, we design a probability-aware reweighting of policy-gradient advantages, upweighting contributions from low-probability tokens to focus learning on reasoning-critical positions as $\Delta\log p$ indicated.
This reweighting yields additional gains over current state-of-the-art RLVR methods (\eg DAPO \citep{DAPO}) across diverse benchmarks and models.

In summary, this work introduces a directional diagnostic for analyzing RLVR's effects and, based on these findings, develops two practical strategies for reasoning enhancement: a test-time extrapolation technique and a training-time reweighting method. 
We hope our work offers a new perspective for analyzing and improving RLVR through the lens of update direction.

\vspace{-3pt}
\section{Preliminaries}
\vspace{-3pt}

\textbf{Group Relative Policy Optimization (GRPO).} GRPO \citep{GRPO} is a variant of the milestone policy gradient algorithm PPO \citep{PPO}. It is adapted for LLM training by eliminating the need for a separate critic model.
For each QA pair $(x, a)$ sampled from dataset $\mathcal D$, GRPO generates a group of $G$ responses $\{y_i\}_{i=1}^G$ using the old policy $\pi_{\theta_\text{old}}$, computes their rewards $\{R_i\}_{i=1}^G$, and estimates the advantage of each response in a group-relative manner:
\begin{equation}
    \hat A_{i,t} = \frac{R_i - \mathrm{mean}(\{R_i\}_{i=1}^G)}{\mathrm{std}(\{R_i\}_{i=1}^G)}.
\end{equation}
Then the policy $\pi_\theta$ is optimized by maximizing the following objective:
\begin{equation}\label{eqn:GRPO}
\begin{aligned}
    \mathcal J_\text{GRPO}(\theta) = &\mathbb E_{(x,a)\sim\mathcal{D}, \{y_i\}_{i=1}^G \sim\pi_{\theta_{\text{old}}}(\cdot|x)}\bigg[
     \frac1G\sum_{i=1}^G\frac{1}{|y_i|}\sum_{t=1}^{|y_i|}\min \Bigl( r_{i,t}(\theta)\hat{A}_{i,t}, \\
    & \text{clip}\bigl(r_{i,t}(\theta), 1-\epsilon, 1+\epsilon\bigr)\hat{A}_{i,t} \Bigr) - \beta\mathbb D_\mathrm{KL}(\pi_\theta||\pi_\mathrm{ref}) \bigg],
\end{aligned}
\end{equation}
where $r_{i,t}(\theta) = \frac{\pi_{\theta}(y_{i,t}|x, y_{i,<t})}{\pi_{\theta_{\text{old}}}(y_{i,t}|x, y_{i,<t})}$ is the importance sampling ratio, $\epsilon$ is the clipping range for $r_{i,t}(\theta)$, and $\mathbb D_\mathrm{KL}(\pi_\theta||\pi_\mathrm{ref})$ regularizes the policy to stay close to a reference policy $\pi_\mathrm{ref}$.

\textbf{Dynamic Sampling Policy Optimization (DAPO).} DAPO\citep{DAPO} is a state-of-the-art critic-free RLVR algorithm that further refines GRPO. It introduces several techniques, including clip-higher mechanism, dynamic sampling strategy, token-level loss aggregation, overlong punishment, and removing the KL penalty. DAPO's objective is defined as:
\begin{equation}\label{eqn:DAPO}
\begin{aligned}
    \mathcal{J}_{\text{DAPO}}(\theta) = &\mathbb{E}_{(x,a)\sim\mathcal{D}, \{y_i\}_{i=1}^G \sim\pi_{\theta_{\text{old}}}(\cdot|x)} \bigg[
     \frac{1}{\sum_{i=1}^G |y_i|} \sum_{i=1}^G \sum_{t=1}^{|y_i|} \min \left( r_{i,t}(\theta)\hat{A}_{i,t}, \right. \\
    & \left. \text{clip}\bigl(r_{i,t}(\theta), 1-\epsilon_{\text{low}}, 1+\epsilon_{\text{high}}\bigr)\hat{A}_{i,t} \right) \bigg], \quad \st 0 < \left|\{y_i \mid \text{is\_equivalent}(a, y_i)\}\right| < G.
\end{aligned}
\end{equation}
Given its success, we adopt DAPO as the primary baseline algorithm for our empirical analysis.

\textbf{Token-level metrics for RLVR analysis.} To study how RLVR turns a base model into the RL-finetuned counterpart, we mainly compare the following token-level metrics for RLVR analysis:
\begin{itemize}[left=0pt, topsep=0pt]
\item Entropy: \citet{entropy_28_rule} observed that RLVR-induced changes are sparse and tend to concentrate on high-entropy tokens. This token-level entropy is defined as:
\begin{equation}\label{eqn:entropy}
    \mathcal H_\pi(\cdot|x,y_{<t}) = \mathbb E_{y_t\sim\pi(\cdot|x,y_{<t})}[-\log\pi(y_t|x,y_{<t})].
\end{equation}
We calculate this entropy for both the RLVR model $\mathcal H_{\pi_\mathrm{RL}}$ and the base model $\mathcal H_{\pi_\mathrm{Base}}$.
\item Divergences: \citet{UniReason} used KL Divergence to quantify the distributional shift, also finding that the changes are sparse.  The token-level KL divergence is defined as:
\begin{equation}\label{eqn:divergence}
    \mathbb D^\mathrm{KL}_{\pi_\mathrm{RL},\pi_\mathrm{Base}}(\cdot|x,y_{<t}) = \mathbb E_{y_t\sim\pi_\mathrm{RL}(\cdot|x,y_{<t})}\left[\log\frac{\pi_\mathrm{RL}(y_t|x,y_{<t})}{\pi_\mathrm{Base}(y_t|x,y_{<t})}\right].
\end{equation}
We also include its reversed variant  $\mathbb D^\mathrm{KL}_{\pi_\mathrm{Base},\pi_\mathrm{RL}}$ and the averaged KL Divergence $\mathbb D^\mathrm{KL} = \frac{1}{2}(\mathbb D^\mathrm{KL}_{\pi_\mathrm{RL},\pi_\mathrm{Base}} + \mathbb D^\mathrm{KL}_{\pi_\mathrm{Base},\pi_\mathrm{RL}})$ to avoid asymmetry bias for a comprehensive analysis.
\end{itemize}

\section{Dissecting the Token-Level Changes Introduced by RLVR}\label{sec:analysis}

This section aims to dissect the token-level mechanisms through which RLVR training transforms a base model into its fine-tuned counterpart.
First, we show that the logp difference ($\Delta\log p$, Eq.~\ref{eqn:logp_diff}) captures directional shifts in probability mass and separates base from RLVR generations, whereas magnitude-only metrics (entropy/divergence) do not.
Second, we conduct a token replacement experiment to validate that $\Delta\log p$ more precisely identifies sparse, reasoning-critical tokens targeted by RLVR.
Finally, we explain the sparsity through a gradient analysis showing that policy updates concentrate on low‑probability tokens of RLVR's policy gradient updates.

\subsection{Statistical Analysis: Directional vs. Magnitude-Based Metrics}

\textbf{Experimental Setup.}
We conduct a statistical analysis on outputs from several RLVR-base model pairs (ORZ \citep{ORZ}, DAPO \citep{DAPO}, UniReason \citep{UniReason}) to compare how different token-level metrics capture RLVR-induced changes. 
We plot histograms of entropy, divergences, and logp difference of different models' generated tokens on the AIME-24 dataset.

\textbf{Statistical Comparison.}
Fig.~\ref{fig:teaser}\textcolor{red}{\hyperref[fig:teaser]{(b)}} shows the distributions of these metrics for the UniReason model pair.
Across all metrics, the histograms are sharply peaked near zero (note the log‑scale y‑axis), confirming that RLVR‑induced changes are sparse.\footnote{\citet{entropy_28_rule} argue that RLVR primarily modifies tokens with high entropy. The observed concentration of near‑zero‑entropy tokens is therefore consistent with sparse updates under their assumptions.}
However, the entropy and KL divergence distributions are nearly identical for both the base and RLVR model outputs.
In contrast, the $\Delta\log p$ distribution exhibits two distinct tails: a positive tail corresponding to tokens favored by the RLVR model and a negative tail for the base model. 
This pattern holds across all tested model pairs and for multiple entropy/divergence variants (Appx.~\ref{appdix:stat}): the distributions of magnitude-based metrics are nearly indistinguishable between tokens generated by the RLVR and base models 
(Figs.~\ref{fig:hist_unireason_magnitudes}-\ref{fig:hist_orz_magnitudes}), whereas $\Delta\log p$ consistently exhibits clear bimodal patterns (Fig.~\ref{fig:hist_logp_diffs}).

This is because magnitude‑only metrics quantify the size of the distributional change but \textbf{ignore its direction,} \ie whether a given token is more favored by the RLVR model or the base model. 
With directional information, $\Delta\log p$ \textbf{reveals a clear difference between the two modes}, enabling more precise identification of the sparse, reasoning‑enhancing updates induced by RLVR, and we will validate their impact on reasoning performance in the following section.

\subsection{Recovering RLVR Performance via Selective Token Replacement}

\begin{figure}[t]
    \centering
    \includegraphics[width=\linewidth]{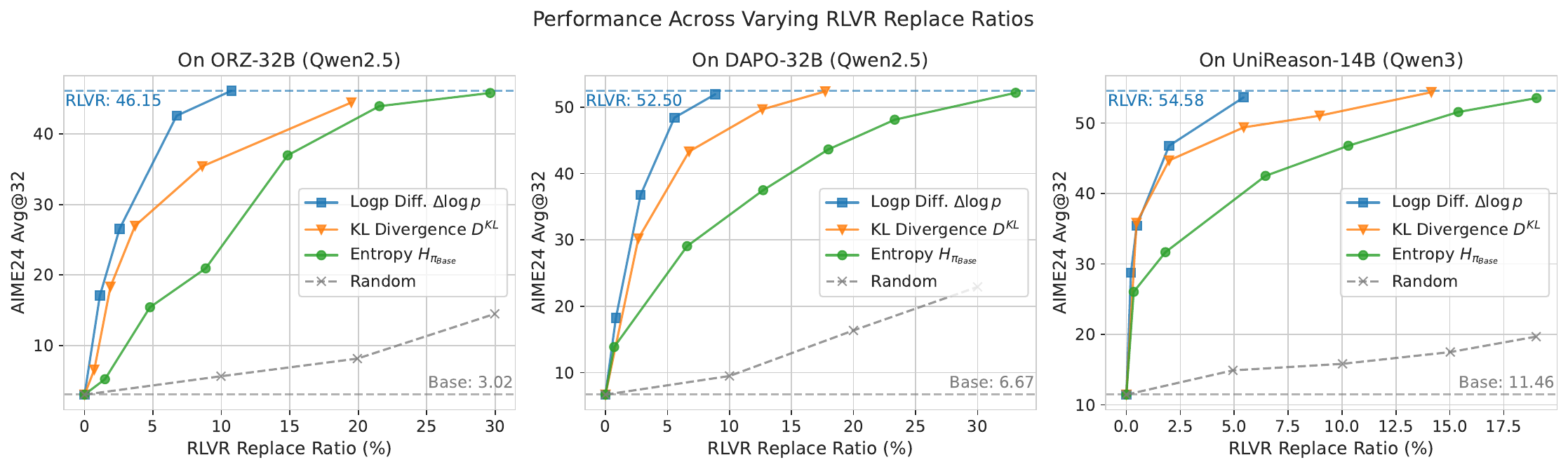}
    \caption{Token‑replacement performance across metrics and model pairs. 
    While all metrics can recover RLVR‑level accuracy, $\Delta\log p$ does so with \emph{the fewest replacements}, demonstrating its precision in isolating the reasoning-critical minor tokens changed by RLVR training.}
    \label{fig:replace}
\end{figure}

\textbf{Token Replacement Setup.}
To further assess how the minority tokens identified by each metric affect reasoning ability, we conduct a \emph{selective token replacement}\footnote{This follows the \texttt{cross-sample} experiment by \citet{crosssample}, which originally employs bidirectional token swapping to verify RL's sparsity. 
We use the term \emph{selective token replacement} to better reflect our specific setup: comparing how different metrics select base tokens to be replaced by $\pi_{\mathrm{RL}}$.}
experiment proposed by \citet{crosssample}.
At each decoding step, we sample a token from $\pi_\mathrm{Base}$, then apply a metric-specific criterion $f^\tau$ to decide whether to replace the token with one sampled from  $\pi_\mathrm{RL}$ (Alg.~\ref{alg:replace}). 
The threshold $\tau$ is adjusted to control replacement rates across metrics, enabling fair comparisons.

We compare entropy, KL Divergences\footnote{We mainly use the averaged KL divergence $\mathbb D^\mathrm{KL} = \frac12(\mathbb D^\mathrm{KL}_{\pi_\mathrm{RL},\pi_\mathrm{Base}} + \mathbb D^\mathrm{KL}_{\pi_\mathrm{Base},\pi_\mathrm{RL}})$ for token replacement to avoid potential asymmetry bias and include KL's variants $\mathbb D^\mathrm{KL}_{\pi_\mathrm{RL},\pi_\mathrm{Base}}$ and $\mathbb D^\mathrm{KL}_{\pi_\mathrm{Base},\pi_\mathrm{RL}}$ for ablation study.}, and logp difference, with the corresponding replacement criteria functions defined as follows:

\begin{wrapfigure}[13]{r}{0.48\textwidth}
\vspace{-23pt}
    \begin{minipage}{0.48\textwidth}
        \begin{algorithm}[H]
        \caption{Selective Token Replacement}
        \label{alg:replace}
        \begin{algorithmic}[1]
        \REQUIRE Base and RLVR models $\pi_\mathrm{Base}, \pi_\mathrm{RL}$, prompt $x$, criterion function $f^\tau(\cdot)\in\{0,1\}$
        
        \STATE Initialize response: $t \gets 0, y_{\le 0} \gets \text{``''}$
        
        \WHILE{$y_t \neq \text{``\texttt{<EOS>}''}$}
            \STATE $t \gets t + 1$ 
            \STATE \label{line:sample} Sample from base: $y_t \sim \pi_\mathrm{Base}(\cdot|x, y_{<t})$
            \IF{$f^\tau(y_t|x, y_{<t}) = 1$}
                \STATE \label{line:replace} Replace the token: $y_t \sim \pi_\mathrm{RL}(\cdot|x, y_{<t})$
            \ENDIF
        \ENDWHILE
        
        \RETURN $y_{\le t}$
        \end{algorithmic}
        \end{algorithm}
    \end{minipage}
\end{wrapfigure}

\begin{itemize}[left=0pt, topsep=0pt]
    \item Entropy: Following the hypothesis that RLVR updates target high-entropy positions \citep{entropy_28_rule}, we replace the base model's token if its token distribution has entropy exceeding a threshold $\tau$: $f_\mathcal H^\tau(y_t|x,y_{<t}) = \mathbb I(\mathcal H(\cdot|x,y_{<t})>\tau).$
    \item KL Divergences: Similarly, to target positions where the two models diverges most, we replace the token if the divergence is greater than $\tau$: $f_\mathbb D^\tau(y_t|x,y_{<t}) = \mathbb I\big(\mathbb D(\cdot|x,y_{<t})>\tau\big).$
    \item Logp Difference: A large negative $\Delta\log p$ for a token $y_t$ indicates that RLVR has learned to penalize it relative to the base model. We exploit this by replacing tokens whose logp difference falls \textit{below} a threshold $\tau$: $f_\mathrm{logp}^\tau(y_t|x,y_{<t}) = \mathbb I\big(\Delta\log p(y_t|x,y_{<t})<\tau\big).$
\end{itemize}

This selective replacement setup, controlled by the metric-specific thresholds, allows us to compare the impact of tokens identified by each metric on reasoning performance at matched replacement rates.
Fig.~\ref{fig:replace} shows results on AIME‑24 for three representative metrics $\mathcal H_{\pi_\mathrm{Base}}$, $\mathbb D^{\mathrm{KL}}$, and $\Delta\log p$, while Fig.~\ref{fig:kl_ablation} in Appx.~\ref{appdix:replace_abl} provides ablations with additional metrics, including the RLVR model's entropy $\mathcal H_{\pi_\mathrm{RL}}$ and KL‑divergence variants. 
All metrics are contrasted with a random baseline that uniformly replaces tokens: $f^\tau_\mathrm{rand}(\cdot) = \mathbb I_{\rho\sim U[0,1]}(\rho<\tau)$.
The key observations are as follows:

\textbf{Observation I: Selectively replacing a minority of base models' tokens can recover RLVR performance.}
As shown in Fig.~\ref{fig:replace}, replacing 5-30\% of a base model's sampled tokens with different metrics suffices to match the final RLVR model's accuracy.
In contrast, randomly replacing the tokens without metric selection produces much slower performance growth. 
This demonstrates that RLVR‑modified tokens are sparsely distributed along the sequence but disproportionately important for reasoning, highlighting the efficacy of the evaluated metrics in identifying these critical tokens.

\textbf{Observation II: Logp difference $>$ divergence $>$ entropy in identifying RLVR-learned reasoning patterns.}
Across all model pairs (Fig.~\ref{fig:replace}), $\Delta\log p$-based replacement reaches the RLVR model’s accuracy with the \emph{fewest} substitutions (around \emph{10\%} of tokens).
In comparison, magnitude-only metrics (\eg divergence and entropy) require clearly more replacement to match RLVR performance, indicating lower precision in identifying reasoning‑critical changes introduced by RLVR.
Between these two, divergence consistently outperforms entropy, suggesting that RLVR changes may not be restricted to high‑entropy positions.
This ordering---$\Delta \log p$ highest, followed by divergence, then entropy---remains stable across different divergence and entropy variants (Fig.~\ref{fig:kl_ablation} in Appx.~\ref{appdix:replace_abl}), further validating the superiority of logp difference in isolating the most influential positions.

\subsection{A Gradient-Based Explanation for the Sparse Updates}

\begin{figure}[t]
\centering
\begin{subfigure}[b]{0.38\textwidth}
    \centering
    \includegraphics[width=\linewidth]{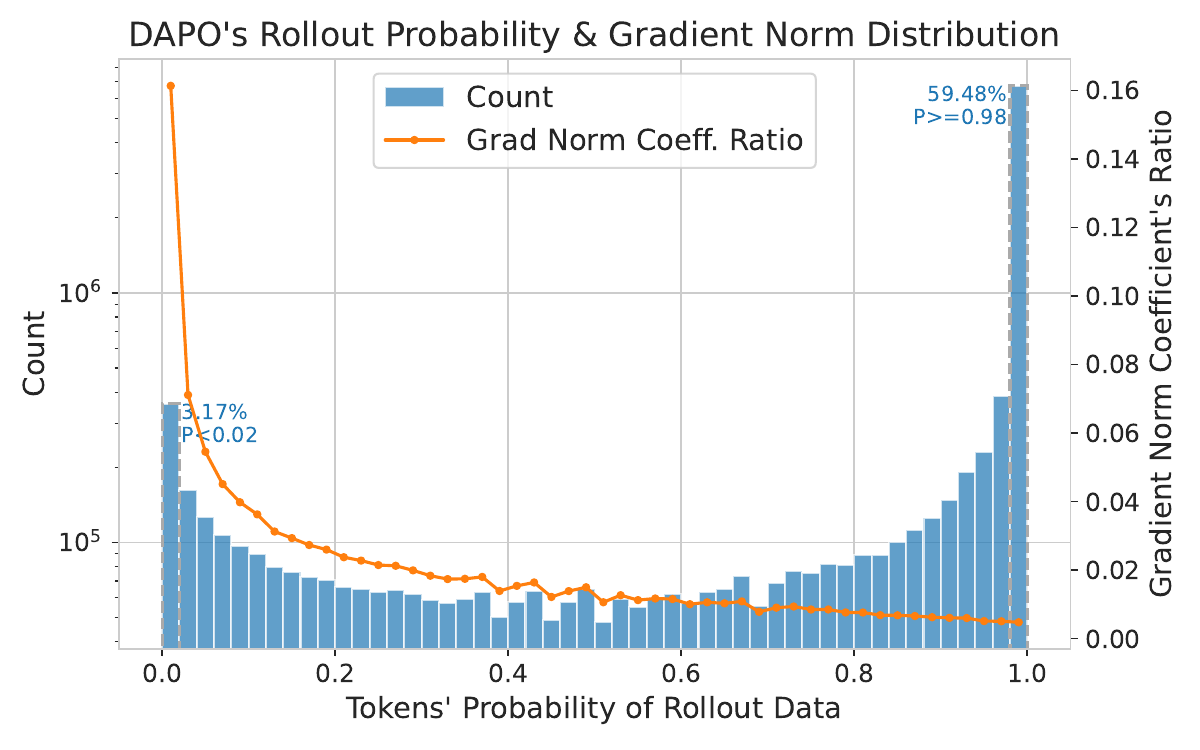}
    \caption{Gradient norm and probability}
\end{subfigure}
\hfill
\begin{subfigure}[b]{0.3\textwidth}
    \centering
    \includegraphics[width=\linewidth]{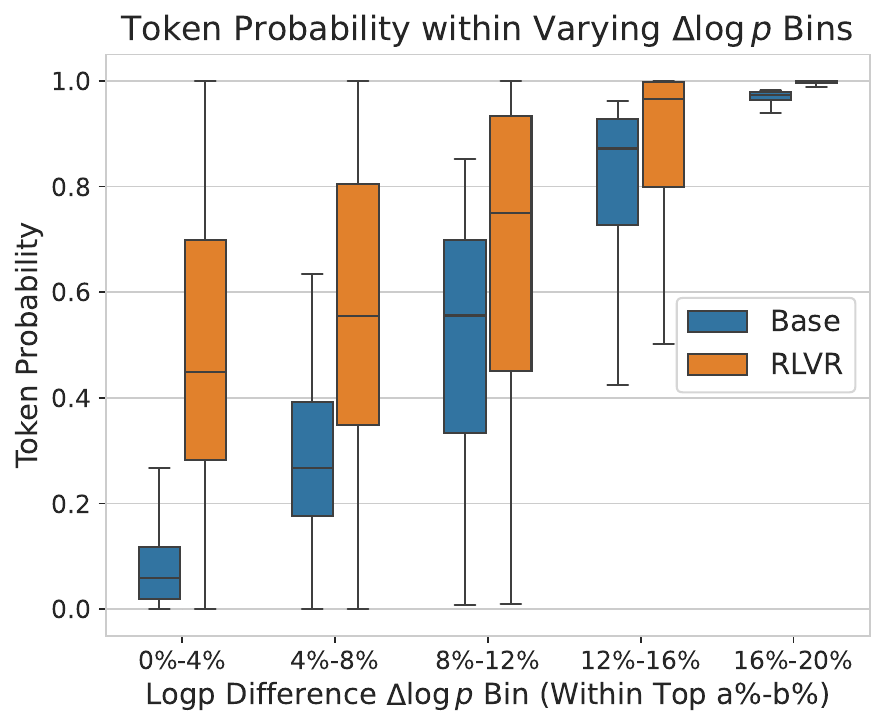}
    \caption{Token probability \textit{v.s.} $\Delta\log p$}
\end{subfigure}
\hfill
\begin{subfigure}[b]{0.3\textwidth}
    \centering
    \includegraphics[width=\linewidth]{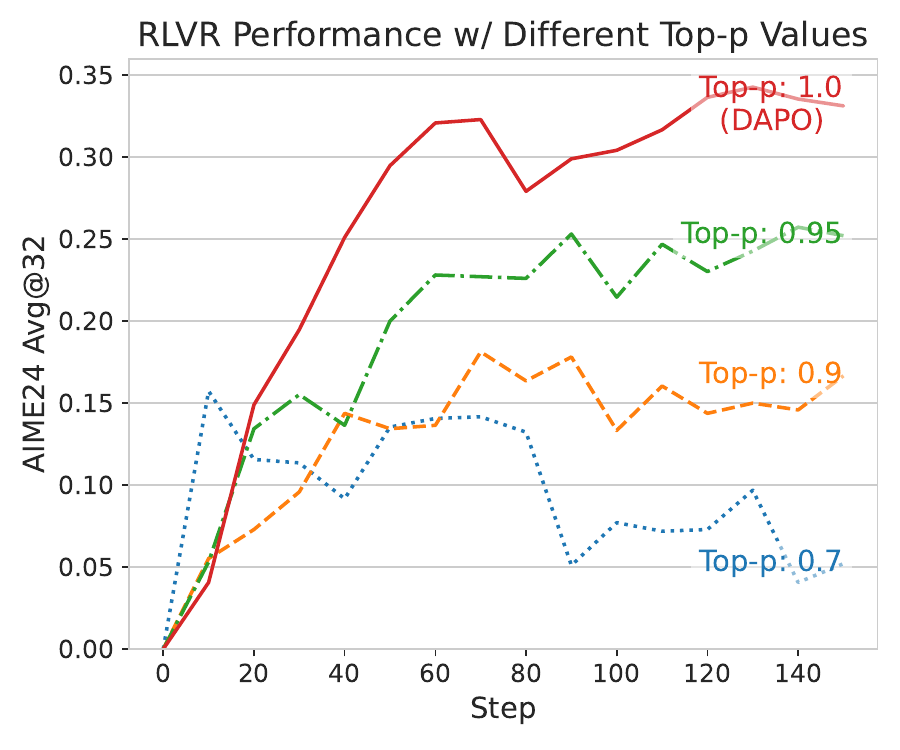}
    \caption{RLVR performance \textit{v.s.} top-p}
\end{subfigure}
\caption{(a) Token probability and gradient norm coefficient $1-\pi_\theta(\cdot)$ of a DAPO step, where the gradient concentrates on rare, low-probability tokens. (b) Token probability within different $\Delta\log p$ bins, where higher $\Delta\log p$ bins contain lower probability for both base and RLVR models. (c) Effect of top-p filtering on RLVR training performance. Performance declines with more filtering.}
\label{fig:probs}
\end{figure}

Our previous analysis established that the RLVR model differs from its base counterpart on a small but critical subset of tokens most effectively identified by $\Delta\log p$.
Here, we provide a gradient-based explanation for this sparsity of changes: RLVR's policy gradient inherently concentrates updates on rare, low-probability tokens, correlating with tokens with high $\Delta\log p$ in the final model.

\textbf{RLVR's policy gradient sparsely concentrates on low-probability tokens.}
The gradient of the DAPO objective $\mathcal J_\mathrm{DAPO}$ for an un-clipped token $y_{i,t}$ can be written as $w_{i,t}\cdot\nabla_\theta\log\pi_\theta(y_{i,t}|x,y_{i,<t})$, where $w_{i,t}=r_{i,t}(\theta)\hat A_{i,t}$ combines the importance sampling ratio and advantage. 
To analyze the token's gradient norm, we have the following lemma (see the proof in Appx.~\ref{proof:grad}):
\begin{lemma}\label{lemma:grad}
    For a softmax-parameterized LLM policy with logits vector $z$ for the output token $y_{i,t}$, the $\ell1$-norm of the DAPO objective's gradient w.r.t. $z$ is given by:
    $$
    \left\Vert\nabla_z\mathcal J_\mathrm{DAPO}(y_{i,t}|x,y_{i,<t})\right\Vert_1 = 2|w_{i,t}|\cdot\big(1-\pi_\theta(y_{i,t}|x,y_{i,<t})\big).
    $$
\end{lemma}
This partial gradient's $\ell1$-norm directly depends on $1-\pi_\theta(y_i|x,y_{i,<t})$, with larger gradient sizes for lower-probability tokens.
Furthermore, \citet{over_dominate} formally proved that the full gradient norm is tightly bounded by the $1-\pi_\theta(\cdot)$ term.
Consequently, low-probability tokens, despite their rarity, receive disproportionately large gradient updates.
We corroborate this empirically in Fig.~\ref{fig:probs}\textcolor{red}{\hyperref[fig:probs]{(a)}}, which plots tokens' probability and their gradient coefficient from an intermediate DAPO training step.
Although low-probability tokens are sampled infrequently, they account for most of the total gradient mass. This concentration of gradients explains why RLVR's modifications are sparse: learning is naturally focused on a small, high-impact set of low-probability positions.

\textbf{High $\boldsymbol{\Delta\log p}$ tokens are the updated low-probability tokens.}
To complete the argument, we link the low-probability tokens that dominate training updates to the high-$\Delta\log p$ tokens in the final model. 
Fig.~\ref{fig:probs}\textcolor{red}{\hyperref[fig:probs]{(b)}} analyzes tokens grouped by their $\Delta\log p$ values.
It reveals two patterns: first, the probability of tokens in high-$\Delta\log p$ bins increases substantially from the base to the RLVR model; second, these high-$\Delta\log p$ tokens have clearly lower probabilities in both models. 
This confirms that the most significant updates learned by RLVR target those low-probability tokens, and the sparsity of RLVR's changes is therefore a direct consequence of sparse, high-magnitude gradients acting on these critical tokens, which can be effectively identified post-hoc by their large $\Delta\log p$.

\textbf{Excluding low-probability tokens during training impairs performance.}
To causally verify the importance of these low-probability tokens, we conduct a training-time intervention experiment to provide direct evidence for our hypothesis.
We train the Qwen2.5-Math-7B base model \citep{Qwen25Math} using DAPO but adopt a top-p sampling strategy during rollout to filter out low-probability tokens.
The results, plotted in Fig.~\ref{fig:probs}\textcolor{red}{\hyperref[fig:probs]{(c)}}, are conclusive. 
Even a mild filter (\eg top-p=0.95) leads to a substantial drop in performance compared to the default setting (top-p=1.0). 
As the filter becomes more aggressive (\ie with lower top-p thresholds), performance degrades sharply.
This experiment demonstrates that these low-probability tokens are not merely correlated with gradient size but are essential for the reasoning improvements achieved by RLVR training.

\begin{tcolorbox}[colback=blue!5,colframe=NavyBlue!70,title=\textbf{Takeaway}]
\begin{enumerate}[left=0pt,topsep=0pt]
    \item \textbf{RLVR's gains stem from sparse, high-impact modifications.} 
    Our analysis reveals that RLVR's performance gains originate not from a global distribution shift, but from targeted, high-impact changes to a minority of tokens.
    \item \textbf{Logp difference pinpoints these sparse changes.}
    By capturing the direction of probability shifts from base to RLVR, logp difference outperforms magnitude-only metrics like entropy or divergence in isolating the reasoning-critical tokens that RLVR learns.
    \item \textbf{Sparsity originates from RLVR's focus on low-probability tokens.}
    The sparse difference is explained by the inherent concentration of RLVR's gradients on rare, low-probability tokens, making these tokens the focal point for improvement and the source of the sparse, high-$\Delta\log p$ changes we observe.
\end{enumerate}
\end{tcolorbox}

\section{Exploiting RLVR's Directional Updates to Boost Reasoning}
Building on Sec.~\ref{sec:analysis}, which isolates sparse and directional updates via $\Delta\log p$, we propose two practical strategies to utilize this directional learning:
(i) a \textit{test-time selective extrapolation} that shifts probability mass further along the learned direction on critical tokens;
(ii) a \textit{training-time advantage reweighting} that prioritizes low-probability tokens implicated by high $\Delta\log p$.
Both methods provide practical ways to boost performance by exploiting the directional mechanisms of RLVR.

\subsection{Test-Time Enhancement via Extrapolation}

\textbf{Selective test-time extrapolation along the $\boldsymbol{\Delta\log p}$ direction.} 
Our token replacement experiment demonstrated that $\Delta\log p$ effectively identifies the reasoning-critical changes of RLVR. 
This raises a natural question: Can we move beyond simple replacement and actively amplify these critical changes to surpass the RLVR model's performance?
We therefore instantiate a token-level extrapolation: treat
$\Delta\log p = \log\pi_\mathrm{RL}(\cdot)-\log\pi_\mathrm{Base}(\cdot)$ as a learned ``reasoning direction'' pointing from base to RLVR distribution.
Our strategy is to amplify this signal by extrapolating the RLVR model's distribution further along this direction. 
The extrapolated policy $\pi_\mathrm{Extra}^\gamma$ is given by:
\begin{equation}\label{eqn:extra}
\begin{aligned}
    \log\pi_\mathrm{Extra}^\gamma(y_t|x,y_{<t}) &\text{ := } \log\pi_\mathrm{RL}(y_t|x,y_{<t}) + \gamma\cdot \Delta\log p(y_t|x,y_{<t}) + z(x,y_{<t}) \\
    &= (1+\gamma)\cdot\log\pi_\mathrm{RL}(y_t|x,y_{<t}) - \gamma\cdot\log\pi_\mathrm{Base}(y_t|x,y_{<t}) + z(x,y_{<t}),
\end{aligned}
\end{equation}
where $\gamma$ is a hyperparameter controlling the extrapolating strength, and $z(\cdot)$ is a log-partition function.
In probability space, this is equivalent to re-weighting the RLVR distribution:
$$\pi_\mathrm{Extra}^\gamma(y_t|x,y_{<t})\propto\pi_\mathrm{RL}(y_t|x,y_{<t})\cdot\exp\big(\gamma\ \Delta\log p(y_t|x,y_{<t})\big).$$
This framing connects our method to reward-guided decoding literature \citep{decoding:ARGS, decoding:DeRa, decoding:GenARM}, where a reward function is used to re-weight the probability distribution.
Our $\Delta\log p$ thereby acts as a token-level reward that encourages better reasoning in this framework.

\textbf{Why selective?} RLVR’s improvements concentrate on a minority of tokens; most positions exhibit negligible $\Delta\log p$.
A global intervention risks distorting well-calibrated tokens.
We therefore apply extrapolation \emph{selectively}, using $f^\tau_\mathrm{logp}$ to gate positions with large negative $\Delta\log p$, and sample from the extrapolated policy $\pi^{\gamma}_{\mathrm{extra}}$ only at those positions (substituting $\pi_\mathrm{RL}$ in Algo.~\ref{alg:replace}, Line~\ref{line:replace}).


\textbf{Empirical Setup.}
We evaluate our method on the AIME-24 benchmark using the ORZ, DAPO, and UniReason model pairs, generating 32 samples per question (see Appx.~\ref{appdix:replace_impl} for more details).
To isolate the impact of our strategy, we compare three approaches: (1) RLVR: The original, non-intervened RLVR model $\pi_\mathrm{RL}$; (2) Selective Replace: Base model with tokens replaced by $\pi_\mathrm{RL}$; (3) Selective Extrapolate: Base model with tokens replaced by $\pi_\mathrm{Extra}^\gamma$.
For a controlled comparison, we use the same selection criteria for (2) and (3), with the only difference being the extrapolation.

\begin{wrapfigure}[11]{r}{0.4\textwidth}
    \centering
    \vspace{-12pt}
    \includegraphics[width=\linewidth]{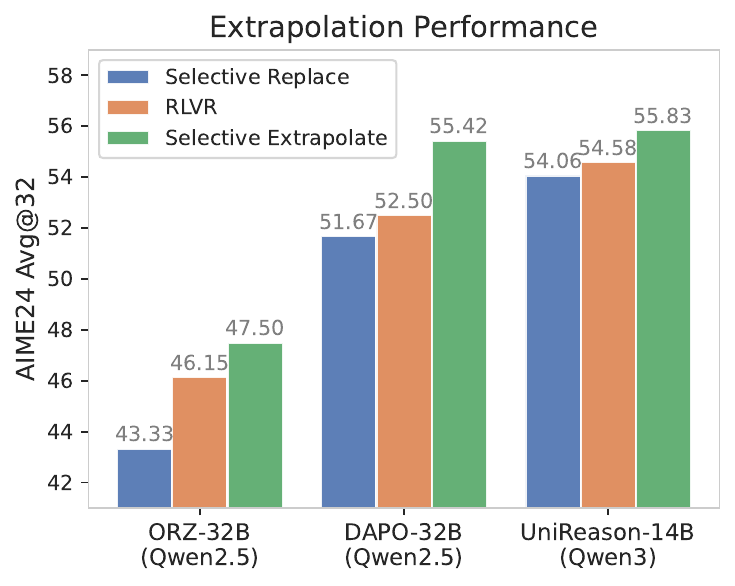}
    \captionsetup{skip=4pt}
    \caption{Extrapolation Performance}
    \label{fig:extrapolate}
\end{wrapfigure}

\textbf{Results.}
On AIME-24, Selective Extrapolation yields higher Avg@32 (average of 32 samples) than $\pi_\mathrm{RL}$ across ORZ-32B, DAPO-32B, and UniReason-14B under matched gates (Fig.~\ref{fig:extrapolate}).
In contrast, Selective Replace matches but does not surpass the RL baseline under the same criteria.
These results indicate that moving beyond $\pi_\mathrm{RL}$ along $\Delta\log p$ provides incremental gains in reasoning accuracy.


\begin{wraptable}[8]{r}{0.4\textwidth}
\vspace{-8pt}
\caption{Selective Extrapolate ($\gamma = 0.1$) on the RLVR model (DAPO-32B) instead of the base model.}
\label{tab:rl_extra}
\centering
\vspace{-6pt}
\scalebox{0.895}{\setlength{\tabcolsep}{0.9mm}{
    \begin{tabular}{@{}ccccc@{}}
    \toprule
    \textbf{Replace Ratio} & 0.0\% & 1.8\% & 5.2\% & 20.0\% \\ \midrule
    \textbf{Avg@32} & 52.50 & 53.96 & \textbf{55.31} & {\ul 55.10} \\
    \textbf{Threshold $\tau$} & N/A & -0.5 & -0.2 & 0.0 \\ \bottomrule
    \end{tabular}
}}
\end{wraptable}

\textbf{Extrapolating on $\boldsymbol{\pi_\mathrm{RL}}$}.
We also apply selective extrapolation directly on $\pi_\mathrm{RL}$ rather than on $\pi_\mathrm{Base}$ in Algo.~\ref{alg:replace} (Line~\ref{line:sample}).
As the threshold $\tau$ in $f_\mathrm{logp}^\tau$ increases, the AIME-24 performance improves up to a moderate intervention ratio, after which gains plateau (Table~\ref{tab:rl_extra}).
This pattern aligns with the sparsity finding: amplifying a limited set of reasoning-critical tokens is effective, whereas aggressive interventions yield diminishing returns.


\textbf{Theoretical Justification.}
Following a standard simplification in theoretical analysis for LLM RL training \citep{NLHF, crucial_sample, MAP}, we consider a tabular softmax bandit policy: $\pi_\theta(y|x)\propto \exp(\theta_{x,y})$, where the logit is individually parameterized by $\theta$ for each prompt-response pair $(x,y)$.
We assume the policy is trained with Natural Policy Gradient (NPG \citep{NPG}) following \citet{entropy_mechanism}, since its updates resemble the controlled optimization of PPO \citep{PPO}.
The update rule of NPG via backtracking simplifies to: $\theta^{t+1}_{x,y}-\theta_{x,y}^t = \eta\cdot A^t(x,y)$, where $\eta$ is the step size and $A^t$ is the advantage function \citep{PGTheory}.
In this context, our extrapolated policy (Eq.~\ref{eqn:extra}) is defined as $\pi_{\omega(\theta^t,\gamma)}$, where $\omega(\theta^t,\gamma) = \theta^t + \gamma(\theta^t-\theta^0)$.
Under these conditions, we have the following theorem (the proof can be found in Appx.~\ref{proof:extra}):
\begin{theorem}\label{theorem:extra}
    For a given prompt $x$, if a tabular softmax policy $\pi_{\theta^t}$ is updated via natural policy gradient \citep{NPG}, then the extrapolated policy $\pi_{\omega(\theta^t,\gamma)}$ satisfies:
    $$
    \exists\ \gamma>0, \mathbb E_{y\sim \pi_{\omega(\theta^t,\gamma)}(\cdot|s)}[R_{x,y}] \ge \mathbb E_{y\sim \pi_{\theta^t}(\cdot|s)}[R_{x,y}].
    $$
    Equality holds if and only if the reward $R_{x,y}$ is constant for all $y$.
\end{theorem}

This theorem shows that, in the simplified setting, extrapolating along the learned difference direction of $\Delta\log p$ can improve the expected reward.
Nevertheless, we need to note that the proof relies on the idealized NPG's update rule, with a monotonic learning process consistently adjusting the logits along the reward's direction.
In contrast, our empirical analysis has shown that the updates learned by RLVR concentrate only on a minority of tokens, with $\Delta\log p$ on most tokens being negligible.
This disparity motivates our selective extrapolation only on positions with a significant difference, which exhibit the consistent, directional updates assumed by the theory.

\subsection{Training-Time Enhancement via Advantage Reweighting}

\begin{table}[t]
\caption{Comparison of our reweighting method and DAPO on math reasoning benchmarks.}
\label{tab:main}
\centering
\setlength{\tabcolsep}{0.6mm}{
\begin{tabular}{@{}cccccccccc@{}}
\toprule
\multirow{2}{*}{\textbf{Model}} & \multirow{2}{*}{\textbf{Method}} & \multicolumn{2}{c}{\textbf{AIME24}} & \multicolumn{2}{c}{\textbf{AIME25}} & \multicolumn{2}{c}{\textbf{AMC}} & \multicolumn{2}{c}{\textbf{Average}} \\ \cmidrule(l){3-10} 
 &  & Avg@32 & Pass@16 & Avg@32 & Pass@16 & Avg@32 & Pass@16 & Avg@32 & Pass@16 \\ \midrule
\multirow{3}{*}{\begin{tabular}[c]{@{}c@{}}Qwen2.5-\\ Math-7B\end{tabular}} & Base & 14.79 & 47.46 & 6.67 & 27.84 & 40.62 & 79.25 & 20.69 & 51.52 \\
 & DAPO & 35.73 & 54.09 & 17.6 & 30.45 & 73.04 & 89.03 & 42.12 & 57.86 \\
 & Ours & \textbf{39.06} & \textbf{60.58} & \textbf{18.54} & \textbf{36.72} & \textbf{73.64} & \textbf{89.69} & \textbf{43.75} & \textbf{62.33} \\ \midrule
\multirow{3}{*}{\begin{tabular}[c]{@{}c@{}}Qwen3-\\ 8B-Base\end{tabular}} & Base & 5.42 & 30.63 & 5.73 & 32.8 & 27.64 & 78.09 & 12.93 & 47.17 \\
 & DAPO & 36.98 & \textbf{72.3} & 26.67 & 46.76 & 69.13 & 88.51 & 44.26 & 69.19 \\
 & Ours & \textbf{38.13} & 69.87 & \textbf{31.15} & \textbf{55.38} & \textbf{71.05} & \textbf{92.3} & \textbf{46.78} & \textbf{72.52} \\ \bottomrule
\end{tabular}
}
\end{table}

\textbf{Training-time enhancement via probability-aware advantage reweighting.}
While our test-time approach amplifies the learned reasoning signal post-hoc, our training-time strategy proactively strengthens the model's reasoning signal during learning.
Instead of extrapolating the final logp difference $\Delta\log p$, we leverage the observed correlation between high $\Delta\log p$ and low-probability tokens (Fig.~\ref{fig:probs}\textcolor{red}{\hyperref[fig:probs]{(b)}}), and propose to amplify the learning signal of these critical low-probability tokens.
Since the parameter update is driven by the advantage term $\hat A_{i,t}$ in policy gradient methods, we modify the advantage in DAPO (Eq.~\ref{eqn:DAPO}) to prioritize low-probability tokens:
\begin{equation}\label{eqn:reweight}
    \tilde A_{i,t} = \big[1+\alpha\cdot\big(1-\pi_{\theta_\mathrm{old}}(y_{i,t}|x,y_{i,< t})\big)\big]\cdot \hat A_{i,t},
\end{equation}
where $\alpha$ is a hyperparameter controlling the reweighting strength.
Such a concentration on low-probability tokens also aligns with our top-p experiment in Fig.~\ref{fig:probs}\textcolor{red}{\hyperref[fig:probs]{(c)}}, which finds that low-probability tokens are irreplaceable for RLVR training.

\textbf{Experimental setup.}
We modify only the advantage (Eq.~\ref{eqn:reweight}) in the standard DAPO recipe and keep all other hyperparameters fixed.
We evaluate model performance on three math reasoning benchmarks: AIME-24, AIME-25, and AMC.
Following DAPO's setup, we use top-p=0.7 for sampling during evaluation. 
We report Avg@32 and Pass@16\footnote{With 32 samples, we report the more stable Pass@16 instead of Pass@32 for Pass@k evaluation.}, both computed over 32 samples per problem to ensure a stable estimate of the pass rates \citep{passk}.

\begin{wraptable}[13]{r}{0.5\textwidth}
    \vspace{-12pt}
    \caption{Results of various reweighting methods.}
    \label{tab:compare}
    \centering
    \vspace{-4pt}
    \setlength{\tabcolsep}{1.3mm}{
    \begin{tabular}{@{}cc|ccc@{}}
    \toprule
    \multicolumn{2}{c|}{\textbf{Method}} & PPL & Dominate & Ours \\ \midrule
    \multicolumn{1}{c|}{\multirow{2}{*}{\textbf{AIME24}}} & Avg@32 & 35.63 & {\ul 36.35} & \textbf{39.06} \\
    \multicolumn{1}{c|}{} & Pass@16 & \textbf{61.95} & 55.27 & {\ul 60.58} \\ \midrule
    \multicolumn{1}{c|}{\multirow{2}{*}{\textbf{AIME25}}} & Avg@32 & {\ul 16.46} & 13.02 & \textbf{18.54} \\
    \multicolumn{1}{c|}{} & Pass@16 & {\ul 32.19} & 20.69 & \textbf{36.72} \\ \midrule
    \multicolumn{1}{c|}{\multirow{2}{*}{\textbf{AMC}}} & Avg@32 & 72.06 & \textbf{79.97} & {\ul 73.64} \\
    \multicolumn{1}{c|}{} & Pass@16 & {\ul 89.1} & 84.93 & \textbf{89.69} \\ \midrule
    \multicolumn{1}{c|}{\multirow{2}{*}{\textbf{Average}}} & Avg@32 & 41.38 & {\ul 43.11} & \textbf{43.75} \\
    \multicolumn{1}{c|}{} & Pass@16 & {\ul 61.08} & 53.63 & \textbf{62.33} \\ \bottomrule
    \end{tabular}
    }
\end{wraptable}

\textbf{Results: performance gains across models and datasets.}
We compare our reweighting method on two models: Qwen2.5-Math-7B \citep{Qwen25Math} and Qwen3-8B-Base \citep{Qwen3}.
As shown in Tab.~\ref{tab:main}, enhancing low-probability tokens' weight consistently improves reasoning accuracy across all tested models and datasets.
Notably, this enhanced accuracy (Avg@32) doesn't come at the cost of exploration ability (often measured by Pass@k) \citep{doesrl}; in fact, the average Pass@16 also increases over the DAPO baseline.

\textbf{Comparison of different reweighting.}
While our reweighting method is motivated by the critical role of low-probability tokens, existing work has proposed alternative reweighting strategies that stem from different hypotheses:
(1) PPL: \citet{decomposing_zhaoxin} find that RLVR updates favor \textit{low-ppl} responses, so they reweight advantage to enhance these responses: $\tilde A_{i,t}^\mathrm{ppl} = [1-\alpha\cdot w_\mathrm{ppl}(y_i)]\cdot\hat A_{i,t}$, where $w_\mathrm{ppl}(y_i)$ is a normalized log-PPL weight.
(2) Dominate: \citet{over_dominate} argue that RLVR training can be \textit{over-dominated} by low-probability tokens, so they propose to counteract this by upweighting high-probability tokens: $\tilde A_{i,t}^\mathrm{dom} = [\alpha\cdot\pi_\theta(y_{i,t})+1-\alpha]\cdot\hat A_{i,t}$.
We implement these methods using their recommended hyperparameters and compare the performance on Qwen2.5-Math-7B.
As shown in Table~\ref{tab:compare}, our method of directly amplifying low-probability tokens achieves the best overall performance for both Avg@32 and Pass@16. 
The training dynamics in Fig.~\ref{fig:training_curve} provide further insight: Our method not only exhibits higher reasoning accuracy but also a steady increase in response length.
This simultaneous increase in performance and length is a key pattern in effective reasoning RLVR training \citep{Deepseek-R1}, suggesting the promoted reasoning behavior by our method.
Moreover, the training entropy of $\tilde A_{i,t}^\mathrm{dom}$ reweighting is clearly lower, since they adopt a more restrictive clip-higher ratio of $\epsilon_\mathrm{high} = 0.24$ than the default $\epsilon_\mathrm{high} = 0.28$ in DAPO\footnote{This follows the recommended value in their paper \citep{over_dominate}. We also tested the default $\epsilon_\mathrm{high} = 0.28$, but it resulted in unstable training.}.
The lower entropy (less exploration) also explains their reduced Pass@k performance in Tab.~\ref{tab:compare}.

\begin{figure}[t]
    \centering
    \includegraphics[width=0.9\linewidth]{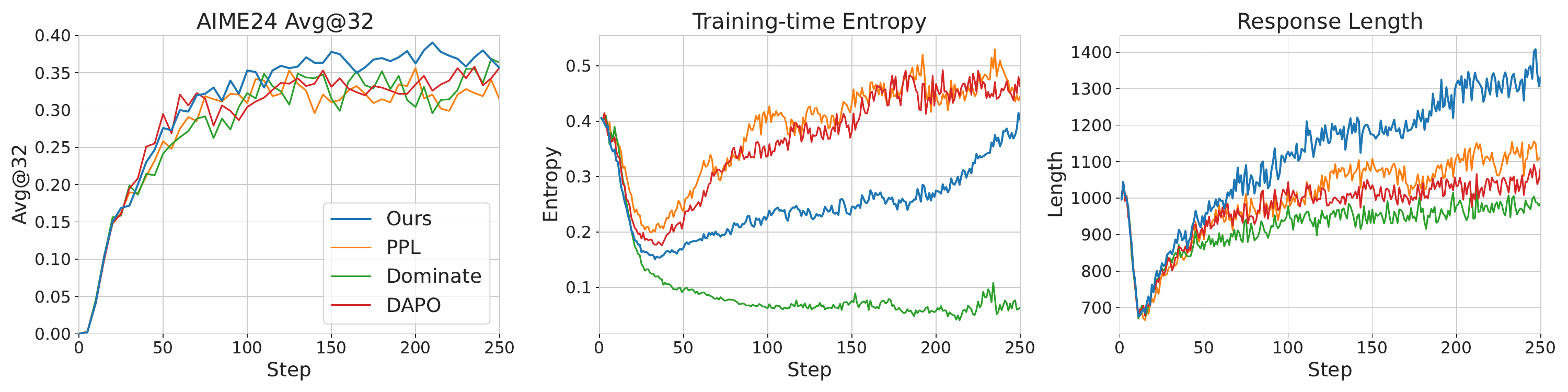}
    \captionsetup{skip=5pt}
    \caption{Training curves for different reweighting methods on Qwen2.5-Math-7B.}
    \vspace{-1.5em}
    \label{fig:training_curve}
\end{figure}
\vspace{-3pt}
\section{Related Work}
\vspace{-3pt}

\textbf{Reinforcement learning for LLM.}
Reinforcement learning is a pivotal component of the LLM post-training pipeline. 
Early applications centered on Reinforcement Learning from Human Feedback (RLHF) for model alignment \citep{rlhf-ouyang2022training, rlhf-stiennon2020learning}, while recent advancements shift their focus to building reasoning models with RL.
OpenAI o1 \citep{Openai-O1} is the first reasoning model, and DeepSeek R1 \citep{Deepseek-R1} introduces a detailed RLVR \citep{tulu-RLVR} recipe for building reasoning models with the GRPO algorithm \citep{GRPO}.
These seminal works inspired the development of a series of subsequent models, from industrial systems like Kimi\citep{kimi-k1.5}, Qwen3 \citep{Qwen3}, and Gemini 2.5 \citep{Gemini-2.5}, to open-source academic algorithms such as Dr.GRPO \citep{Dr.GRPO}, Open-Reasoner-Zero \citep{ORZ}, DAPO \citep{DAPO}, GSPO \citep{GSPO}, and QAE \citep{QAE}, to further improve the reasoning abilities.
In this paper, we adopt DAPO as our baseline RLVR algorithm.

\textbf{Understanding the effects of RLVR.}
The success of RLVR has prompted a line of research dedicated to understanding its effects. While early work analyzed high-level cognitive behaviors of RLVR-trained models \citep{cognitive_behavior, hu2025distillation,thoughtanchors}, recent studies have deepened the analysis with token-level quantification \citep{qian2025demystifying, emergent_hirerarchical}. 
\citet{entropy_mechanism} studied the token entropy change during RLVR, \citet{over_dominate} quantified the gradient norm of specific tokens, and \citet{decomposing_zhaoxin, crosssample} used token replacement to measure their impact on reasoning performance.
A core finding from these analyses is that RLVR induces sparse updates, which have been verified through high-entropy tokens \citep{entropy_28_rule}, KL Divergences \citep{UniReason}, and the sparse gradient norm \citep{over_dominate, decomposing_zhaoxin}.
\updated{However, when studying the differences between base and RLVR models, prior studies mainly focus on the magnitude of changes, largely overlooking their direction.
While \citep{over_dominate} analyzes the update direction (increase or decrease) of probabilities at each gradient step, we extend the notion of update direction to the full distributional shift from the base model to the RLVR model, and we propose explicitly extrapolating along this learned direction in distribution space.
}

\vspace{-3pt}
\section{Conclusion}
\vspace{-3pt}

In this work, we introduced a directional analysis of RLVR based on the logp difference $\Delta\log p$, shown to be more effective in identifying sparse yet reasoning-critical updates than magnitude-based metrics (\eg divergence or entropy).
Building on this, we proposed a test-time extrapolation to amplify these directional updates and a training-time reweighting to focus learning on the low-probability tokens that $\Delta\log p$ highlights.
Both methods improve reasoning performance across different settings, validating our key principle: diagnose and improve RLVR by its update direction.

\textbf{Limitations and future work.}
One primary limitation of our extrapolation method is the requirement of two models; future work could integrate this with parameter-efficient finetuning to reduce computational cost.
The extrapolation also introduces additional hyperparameters, and future work can explore combining the selection threshold and extrapolation strength for a more adaptive extrapolation.
Additionally, our reweighting approach could be evaluated for different model scales or combined with other adaptive training techniques.

\section*{Contributions}\label{sec:contribution}
\textbf{Authors:}
Kexin Huang, Haoming Meng, Junkang Wu, Jinda Lu, Chiyu Ma, Ziqian Chen, Xue Wang, Bolin Ding, Jiancan Wu, Xiang Wang, Xiangnan He, Guoyin Wang, and Jingren Zhou.

\clearpage
\bibliography{bibliography}

@article{Openai-O1,
  title={Openai o1 system card},
  author={Jaech, Aaron and Kalai, Adam and Lerer, Adam and Richardson, Adam and El-Kishky, Ahmed and Low, Aiden and Helyar, Alec and Madry, Aleksander and Beutel, Alex and Carney, Alex and others},
  journal={arXiv preprint arXiv:2412.16720},
  year={2024}
}

@article{Deepseek-R1,
  title     = {DeepSeek-R1 incentivizes reasoning in LLMs through reinforcement learning},
  author    = {Guo, Daya and Yang, Dejian and Zhang, Haowei and Song, Junxiao and Wang, Peiyi and Zhu, Qihao and others},
  journal   = {Nature},
  volume    = {645},
  pages     = {633--638},
  year      = {2025}
}

@article{Gemini-2.5,
  title={Gemini 2.5: Pushing the frontier with advanced reasoning, multimodality, long context, and next generation agentic capabilities},
  author={Comanici, Gheorghe and Bieber, Eric and Schaekermann, Mike and Pasupat, Ice and Sachdeva, Noveen and Dhillon, Inderjit and Blistein, Marcel and Ram, Ori and Zhang, Dan and Rosen, Evan and others},
  journal={arXiv preprint arXiv:2507.06261},
  year={2025}
}

@article{Qwen3,
  title={Qwen3 technical report},
  author={Yang, An and Li, Anfeng and Yang, Baosong and Zhang, Beichen and Hui, Binyuan and Zheng, Bo and Yu, Bowen and Gao, Chang and Huang, Chengen and Lv, Chenxu and others},
  journal={arXiv preprint arXiv:2505.09388},
  year={2025}
}

@article{kimi-k1.5,
  title={Kimi k1. 5: Scaling reinforcement learning with llms},
  author={Team, Kimi},
  journal={arXiv preprint arXiv:2501.12599},
  year={2025}
}

@article{Qwen25,
  title={Qwen2.5 technical report},
  author={Team, Qwen},
  journal={arXiv preprint arXiv:2412.15115},
  year={2024}
}

@misc{Qwen25Math,
      title={Qwen2.5-Math Technical Report: Toward Mathematical Expert Model via Self-Improvement}, 
      author={An Yang and Beichen Zhang and Binyuan Hui and Bofei Gao and Bowen Yu and Chengpeng Li and Dayiheng Liu and Jianhong Tu and Jingren Zhou and Junyang Lin and Keming Lu and Mingfeng Xue and Runji Lin and Tianyu Liu and Xingzhang Ren and Zhenru Zhang},
      year={2024},
      eprint={2409.12122},
      archivePrefix={arXiv},
      primaryClass={cs.CL}
}

@article{tulu-RLVR,
  title={Tulu 3: Pushing frontiers in open language model post-training},
  author={Lambert, Nathan and Morrison, Jacob and Pyatkin, Valentina and Huang, Shengyi and Ivison, Hamish and Brahman, Faeze and Miranda, Lester James V and Liu, Alisa and Dziri, Nouha and Lyu, Shane and others},
  journal={arXiv preprint arXiv:2411.15124},
  year={2024}
}

@article{ORZ,
  title={Open-reasoner-zero: An open source approach to scaling up reinforcement learning on the base model},
  author={Hu, Jingcheng and Zhang, Yinmin and Han, Qi and Jiang, Daxin and Zhang, Xiangyu and Shum, Heung-Yeung},
  journal={arXiv preprint arXiv:2503.24290},
  year={2025}
}

@article{PPO,
  title={Proximal policy optimization algorithms},
  author={Schulman, John and Wolski, Filip and Dhariwal, Prafulla and Radford, Alec and Klimov, Oleg},
  journal={arXiv preprint arXiv:1707.06347},
  year={2017}
}

@article{GRPO,
  title={Deepseekmath: Pushing the limits of mathematical reasoning in open language models},
  author={Shao, Zhihong and Wang, Peiyi and Zhu, Qihao and Xu, Runxin and Song, Junxiao and Bi, Xiao and Zhang, Haowei and Zhang, Mingchuan and Li, YK and Wu, Yang and others},
  journal={arXiv preprint arXiv:2402.03300},
  year={2024}
}

@article{DAPO,
  title={Dapo: An open-source llm reinforcement learning system at scale},
  author={Yu, Qiying and Zhang, Zheng and Zhu, Ruofei and Yuan, Yufeng and Zuo, Xiaochen and Yue, Yu and Dai, Weinan and Fan, Tiantian and Liu, Gaohong and Liu, Lingjun and others},
  journal={arXiv preprint arXiv:2503.14476},
  year={2025}
}

@article{Dr.GRPO,
  title={Understanding r1-zero-like training: A critical perspective},
  author={Liu, Zichen and Chen, Changyu and Li, Wenjun and Qi, Penghui and Pang, Tianyu and Du, Chao and Lee, Wee Sun and Lin, Min},
  journal={arXiv preprint arXiv:2503.20783},
  year={2025}
}

@article{GSPO,
  title={Group Sequence Policy Optimization},
  author={Zheng, Chujie and Liu, Shixuan and Li, Mingze and Chen, Xiong-Hui and Yu, Bowen and Gao, Chang and Dang, Kai and Liu, Yuqiong and Men, Rui and Yang, An and others},
  journal={arXiv preprint arXiv:2507.18071},
  year={2025}
}

@article{GMPO,
  title={Geometric-Mean Policy Optimization},
  author={Zhao, Yuzhong and Liu, Yue and Liu, Junpeng and Chen, Jingye and Wu, Xun and Hao, Yaru and Lv, Tengchao and Huang, Shaohan and Cui, Lei and Ye, Qixiang and others},
  journal={arXiv preprint arXiv:2507.20673},
  year={2025}
}

@article{entropy_mechanism,
  title={The entropy mechanism of reinforcement learning for reasoning language models},
  author={Cui, Ganqu and Zhang, Yuchen and Chen, Jiacheng and Yuan, Lifan and Wang, Zhi and Zuo, Yuxin and Li, Haozhan and Fan, Yuchen and Chen, Huayu and Chen, Weize and others},
  journal={arXiv preprint arXiv:2505.22617},
  year={2025}
}

@article{entropy_28_rule,
  title={Beyond the 80/20 rule: High-entropy minority tokens drive effective reinforcement learning for llm reasoning},
  author={Wang, Shenzhi and Yu, Le and Gao, Chang and Zheng, Chujie and Liu, Shixuan and Lu, Rui and Dang, Kai and Chen, Xionghui and Yang, Jianxin and Zhang, Zhenru and others},
  journal={arXiv preprint arXiv:2506.01939},
  year={2025}
}

@article{over_dominate,
  title={Do Not Let Low-Probability Tokens Over-Dominate in RL for LLMs},
  author={Yang, Zhihe and Luo, Xufang and Wang, Zilong and Han, Dongqi and He, Zhiyuan and Li, Dongsheng and Xu, Yunjian},
  journal={arXiv preprint arXiv:2505.12929},
  year={2025}
}

@article{UniReason,
  title={Does Math Reasoning Improve General LLM Capabilities? Understanding Transferability of LLM Reasoning},
  author={Huan, Maggie and Li, Yuetai and Zheng, Tuney and Xu, Xiaoyu and Kim, Seungone and Du, Minxin and Poovendran, Radha and Neubig, Graham and Yue, Xiang},
  journal={arXiv preprint arXiv:2507.00432},
  year={2025}
}

@misc{decomposing_zhaoxin,
      title={Decomposing the Entropy-Performance Exchange: The Missing Keys to Unlocking Effective Reinforcement Learning}, 
      author={Jia Deng and Jie Chen and Zhipeng Chen and Wayne Xin Zhao and Ji-Rong Wen},
      year={2025},
      eprint={2508.02260},
      archivePrefix={arXiv},
      primaryClass={cs.CL},
}

@inproceedings{iclr25_dynamics,
    title={Learning Dynamics of {LLM} Finetuning},
    author={Yi Ren and Danica J. Sutherland},
    booktitle={The Thirteenth International Conference on Learning Representations},
    year={2025}
}

@misc{passk,
      title={Evaluating Large Language Models Trained on Code}, 
      author={Mark Chen and Jerry Tworek and Heewoo Jun and Qiming Yuan and Henrique Ponde de Oliveira Pinto and Jared Kaplan and Harri Edwards and Yuri Burda and Nicholas Joseph and Greg Brockman and Alex Ray and others},
      year={2021},
      eprint={2107.03374},
      archivePrefix={arXiv},
      primaryClass={cs.LG},
}

@inproceedings{
    decoding:ARGS,
    title={{ARGS}: Alignment as Reward-Guided Search},
    author={Maxim Khanov and Jirayu Burapacheep and Yixuan Li},
    booktitle={The Twelfth International Conference on Learning Representations},
    year={2024},
    url={https://openreview.net/forum?id=shgx0eqdw6}
}

@InProceedings{decoding:DeRa,
  title = 	 {Decoding-time Realignment of Language Models},
  author =       {Liu, Tianlin and Guo, Shangmin and Bianco, Leonardo and Calandriello, Daniele and Berthet, Quentin and Llinares-L\'{o}pez, Felipe and Hoffmann, Jessica and Dixon, Lucas and Valko, Michal and Blondel, Mathieu},
  booktitle = 	 {Proceedings of the 41st International Conference on Machine Learning},
  pages = 	 {31015--31031},
  year = 	 {2024},
  volume = 	 {235},
  series = 	 {Proceedings of Machine Learning Research},
  month = 	 {21--27 Jul},
  publisher =    {PMLR}
}

@inproceedings{
    decoding:GenARM,
    title={Gen{ARM}: Reward Guided Generation with Autoregressive Reward Model for Test-Time Alignment},
    author={Yuancheng Xu and Udari Madhushani Sehwag and Alec Koppel and Sicheng Zhu and Bang An and Furong Huang and Sumitra Ganesh},
    booktitle={The Thirteenth International Conference on Learning Representations},
    year={2025},
    url={https://openreview.net/forum?id=J0qTpmbSbh}
}

@inproceedings{NLHF,
  title={Nash Learning from Human Feedback},
  author={Munos, Remi and Valko, Michal and Calandriello, Daniele and Azar, Mohammad Gheshlaghi and Rowland, Mark and Guo, Zhaohan Daniel and Tang, Yunhao and Geist, Matthieu and Mesnard, Thomas and Fiegel, C{\^o}me and others},
  booktitle={Forty-first International Conference on Machine Learning},
  year={2024}
}

@inproceedings{
    crucial_sample,
    title={The Crucial Role of Samplers in Online Direct Preference Optimization},
    author={Ruizhe Shi and Runlong Zhou and Simon Shaolei Du},
    booktitle={The Thirteenth International Conference on Learning Representations},
    year={2025},
    url={https://openreview.net/forum?id=F6z3utfcYw}
}

@inproceedings{NPG,
 author = {Kakade, Sham M},
 booktitle = {Advances in Neural Information Processing Systems},
 editor = {T. Dietterich and S. Becker and Z. Ghahramani},
 pages = {},
 publisher = {MIT Press},
 title = {A Natural Policy Gradient},
 volume = {14},
 year = {2001}
}

@article{PGTheory,
  author  = {Alekh Agarwal and Sham M. Kakade and Jason D. Lee and Gaurav Mahajan},
  title   = {On the Theory of Policy Gradient Methods: Optimality, Approximation, and Distribution Shift},
  journal = {Journal of Machine Learning Research},
  year    = {2021},
  volume  = {22},
  number  = {98},
  pages   = {1--76},
  url     = {http://jmlr.org/papers/v22/19-736.html}
}

@article{doesrl,
  title={Does Reinforcement Learning Really Incentivize Reasoning Capacity in LLMs Beyond the Base Model?},
  author={Yue, Yang and Chen, Zhiqi and Lu, Rui and Zhao, Andrew and Wang, Zhaokai and Yue, Yang and Song, Shiji and Huang, Gao},
  journal={arXiv preprint arXiv:2504.13837},
  year={2025}
}

@article{rlhf-stiennon2020learning,
  title={Learning to summarize with human feedback},
  author={Stiennon, Nisan and Ouyang, Long and Wu, Jeffrey and Ziegler, Daniel and Lowe, Ryan and Voss, Chelsea and Radford, Alec and Amodei, Dario and Christiano, Paul F},
  journal={Advances in Neural Information Processing Systems (NeurIPS)},
  volume={33},
  pages={3008--3021},
  year={2020}
}

@article{rlhf-ouyang2022training,
  title={Training language models to follow instructions with human feedback},
  author={Ouyang, Long and Wu, Jeffrey and Jiang, Xu and Almeida, Diogo and Wainwright, Carroll and Mishkin, Pamela and Zhang, Chong and Agarwal, Sandhini and Slama, Katarina and Ray, Alex and others},
  journal={Advances in neural information processing systems (NeurIPS)},
  volume={35},
  pages={27730--27744},
  year={2022}
}

@inproceedings{
cognitive_behavior,
title={Cognitive Behaviors that Enable Self-Improving Reasoners, or, Four Habits of Highly Effective {ST}aRs},
author={Kanishk Gandhi and Ayush K Chakravarthy and Anikait Singh and Nathan Lile and Noah Goodman},
booktitle={Second Conference on Language Modeling},
year={2025},
url={https://openreview.net/forum?id=QGJ9ttXLTy}
}

@article{hu2025distillation,
  title={Why Distillation can Outperform Zero-RL: The Role of Flexible Reasoning},
  author={Hu, Xiao and Lu, Xingyu and Mao, Liyuan and Zhang, YiFan and Zhang, Tianke and Wen, Bin and Yang, Fan and Gao, Tingting and Zhou, Guorui},
  journal={arXiv preprint arXiv:2505.21067},
  year={2025}
}

@article{qian2025demystifying,
  title={Demystifying reasoning dynamics with mutual information: Thinking tokens are information peaks in llm reasoning},
  author={Qian, Chen and Liu, Dongrui and Wen, Haochen and Bai, Zhen and Liu, Yong and Shao, Jing},
  journal={arXiv preprint arXiv:2506.02867},
  year={2025}
}

@misc{thoughtanchors,
      title={Thought Anchors: Which LLM Reasoning Steps Matter?}, 
      author={Paul C. Bogdan and Uzay Macar and Neel Nanda and Arthur Conmy},
      year={2025},
      eprint={2506.19143},
      archivePrefix={arXiv},
      primaryClass={cs.LG},
      url={https://arxiv.org/abs/2506.19143}, 
}

@article{emergent_hirerarchical,
  title={Emergent hierarchical reasoning in llms through reinforcement learning},
  author={Wang, Haozhe and Xu, Qixin and Liu, Che and Wu, Junhong and Lin, Fangzhen and Chen, Wenhu},
  journal={arXiv preprint arXiv:2509.03646},
  year={2025}
}

@inproceedings{minerva,
author = {Lewkowycz, Aitor and Andreassen, Anders and Dohan, David and Dyer, Ethan and Michalewski, Henryk and Ramasesh, Vinay and Slone, Ambrose and Anil, Cem and Schlag, Imanol and Gutman-Solo, Theo and Wu, Yuhuai and Neyshabur, Behnam and Gur-Ari, Guy and Misra, Vedant},
title = {Solving quantitative reasoning problems with language models},
year = {2022},
address = {Red Hook, NY, USA},
booktitle = {Proceedings of the 36th International Conference on Neural Information Processing Systems},
articleno = {278},
numpages = {15},
location = {New Orleans, LA, USA},
series = {NeurIPS '22}
}

@article{QAE,
  title={Quantile advantage estimation for entropy-safe reasoning},
  author={Wu, Junkang and Huang, Kexin and Wu, Jiancan and Zhang, An and Wang, Xiang and He, Xiangnan},
  journal={arXiv preprint arXiv:2509.22611},
  year={2025}
}

@inproceedings{
MAP,
title={Larger or Smaller Reward Margins to Select Preferences for {LLM} Alignment?},
author={Kexin Huang and Junkang Wu and Ziqian Chen and Xue Wang and Jinyang Gao and Bolin Ding and Jiancan Wu and Xiangnan He and Xiang Wang},
booktitle={Forty-second International Conference on Machine Learning},
year={2025},
url={https://openreview.net/forum?id=ncTwQagrj8}
}

@inproceedings{
crosssample,
title={Sparse but Critical: A Token-Level Analysis of Distributional Shifts in {RLVR} Fine-Tuning of {LLM}s},
author={Haoming Meng and Kexin Huang and Shaohang Wei and Chiyu Ma and Shuo Yang and Xue Wang and Guoyin Wang and Bolin Ding and Jingren Zhou},
booktitle={The Fourteenth International Conference on Learning Representations},
year={2026},
url={https://openreview.net/forum?id=8vWIXno8LW}
}
\bibliographystyle{colm2024_conference}

\clearpage
\appendix
\section{Selective Token Replacement \& Extraploation}

\subsection{Implementation Details}\label{appdix:replace_impl}
\textbf{Models.}
Our experiments use several publicly available RLVR-trained models and their corresponding base models from the Qwen series \citep{Qwen3, Qwen25}:
\begin{itemize}[left=0pt, topsep=0pt]
\item ORZ: The \href{https://huggingface.co/Open-Reasoner-Zero/Open-Reasoner-Zero-32B}{Open-Reasoner-Zero-32B} model \citep{ORZ}, finetuned from \href{https://huggingface.co/Qwen/Qwen2.5-32B}{Qwen2.5-32B} base model using the PPO algorithm.
\item DAPO: The \href{https://huggingface.co/BytedTsinghua-SIA/DAPO-Qwen-32B}{DAPO-Qwen-32B} model \citep{DAPO}, finetuned from the same \href{https://huggingface.co/Qwen/Qwen2.5-32B}{Qwen2.5-32B} base but with the DAPO algorithm.
\item UniReason: The \href{https://huggingface.co/ReasoningTransferability/UniReason-Qwen3-14B-RL}{UniReason-Qwen3-14B-RL} model \citep{UniReason}, finetuned from \href{https://huggingface.co/Qwen/Qwen3-14B-Base}{Qwen3-14B-Base} using the GRPO algorithm.
\end{itemize}

\textbf{Sampling settings.}
We utilize the \href{https://huggingface.co/datasets/HuggingFaceH4/aime_2024}{AIME-24 dataset} to evaluate the replacement performance.
We adopt the default chat prompt template from each model, with the user prompt defined as follows:
\begin{small}
\begin{verbatim}
[Question]
Please reason step by step, and put your final answer within \\boxed{}.
\end{verbatim}
\end{small}
We set the sampling parameters with top-p=0.7, temperature=1.0, max-length=20k, and sample 32 responses for each question.
The answer is extracted from the last ``boxed'' wrapped text and verified using \href{https://github.com/huggingface/Math-Verify}{Math-Verify}. We report the correctness averaged over 32 samples, \ie Avg@32.

\textbf{Hyperparameters for extrapolation.}
As described in Algo.~\ref{alg:replace}, the replacement is adopted selectively, controlled by the threshold $\tau$ in the criteria function $f^\tau$, while the extrapolation strength is adjusted by the parameter $\gamma$ in $\pi_\mathrm{Extra}^\gamma$.
For the extrapolation results in Fig.~\ref{fig:extrapolate}, the ``Selective Extrpolate'' and ``Selective Replace'' methods share the same hyperparameters for each model, which we summarize as follows:

\begin{table}[htbp]
\centering
\caption{Hyperparameters for the extrapolation results (Fig.~\ref{fig:extrapolate}).}
\label{tab:param_extra}
\begin{tabular}{@{}cccc@{}}
\toprule
\textbf{Model} & \textbf{ORZ} & \updated{\textbf{UniReason}} & \updated{\textbf{DAPO}} \\ \midrule
Threshold $\tau$ for $f_\mathrm{logp}^\gamma$ & -0.4 & -0.35 & -0.3 \\
Replaced Ratio & 10.1\% & 7.5\% & 11.4\% \\
$\gamma$ in $\pi_\mathrm{Extra}^\gamma$ & 0.1 & 0.1 & 0.05 \\ \bottomrule
\end{tabular}
\end{table}

\subsection{Additional Experiments}\label{appdix:replace_abl}

\textbf{Additional metrics.}
As described in Sec.~\ref{sec:analysis}, our primary metrics for token replacement are the base model's entropy $\mathcal H_\mathrm{Base}$, KL Divergence $\mathbb D^\mathrm{KL}$, and logp difference $\Delta\log p$.
For our ablation study, we include additional metrics: the RLVR model's entropy $\mathcal H_\mathrm{RL}$ and two KL-divergence variants: $\mathbb D^\mathrm{KL}_{\pi_\mathrm{RL}, \pi_\mathrm{Base}}$ and $\mathbb D^\mathrm{KL}_{\pi_\mathrm{Base}, \pi_\mathrm{RL}}$.
We evaluate these metrics as criteria for the DAPO model's selective replacement.
By varying the threshold $\tau$ for each criterion, we control the token replacement frequency and plot the performance on AIME-24 against various replacement ratios in Fig.~\ref{fig:kl_ablation}. 
As shown in the figure, although the additional metrics' selected replacements also approach the RLVR model's performance, they still require more replacement than $\Delta\log p$ does.
This confirms the performance ordering for identifying reasoning-critical tokens: logp difference $>$ divergence $>$ entropy.

\begin{figure}[ht]
    \centering
    \includegraphics[width=\linewidth]{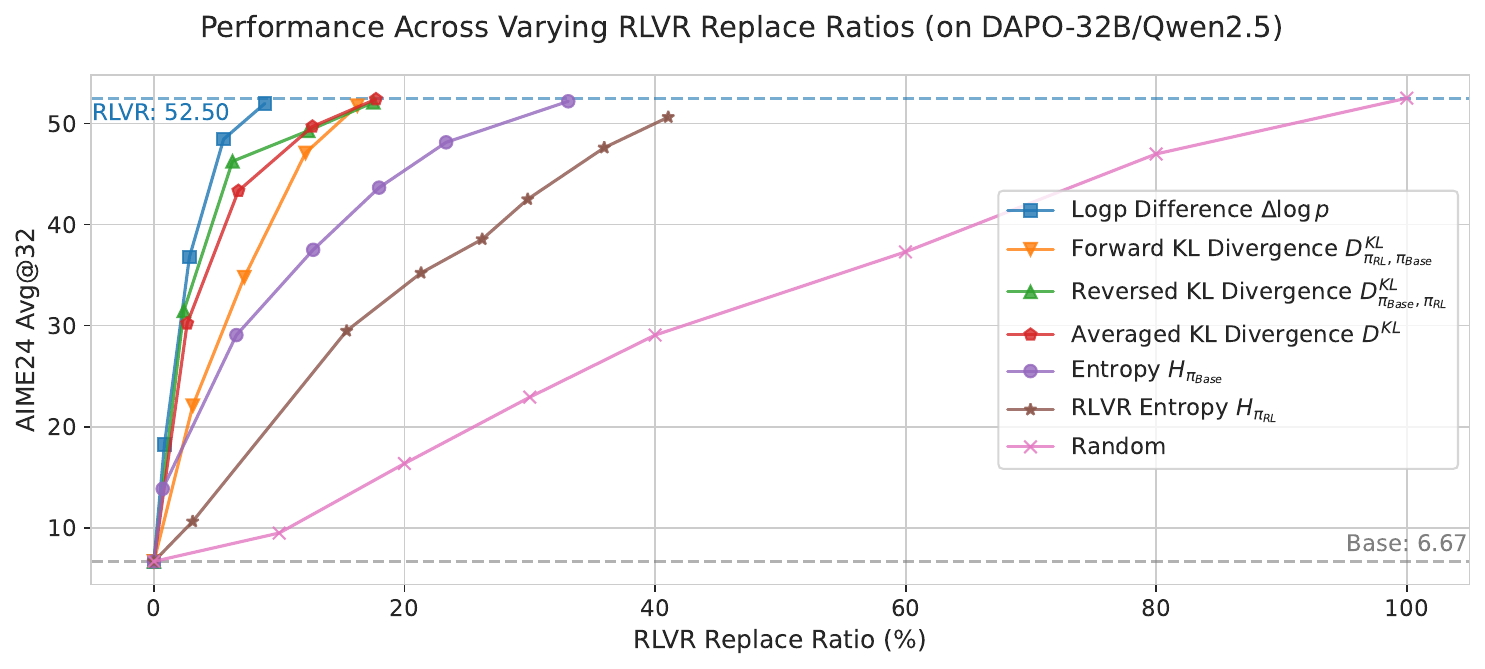}
    \caption{Selective token replacement results with additional criteria for DAPO.}
    \label{fig:kl_ablation}
\end{figure}

\textbf{Selected Tokens.}
To provide an intuitive comparison of the metrics, we analyze the tokens utilized for replacing the base model's choice during DAPO's token replacement of entropy $\mathcal H_{\pi_\mathrm{Base}}$, KL Divergence $\mathbb D^\mathrm{KL}$, and logp difference $\Delta\log p$.
To ensure a fair comparison, we adjust the threshold for each metric to achieve a replacement rate of approximately 8\%. 
Fig.~\ref{fig:top_tokens} illustrates each criterion's top 50 substitution tokens. 
The figure reveals that entropy-based selection favors logical transition words (e.g., Thus, need, can), while the divergence and $\Delta\log p$ criteria utilize more specific mathematical reasoning tokens, including a higher proportion of math symbols.
Combined with the inferior performance of the entropy criterion, this suggests that these specific mathematical tokens might be more efficient for improving reasoning performance.

\begin{figure}[ht]
    \centering
    \includegraphics[width=\linewidth]{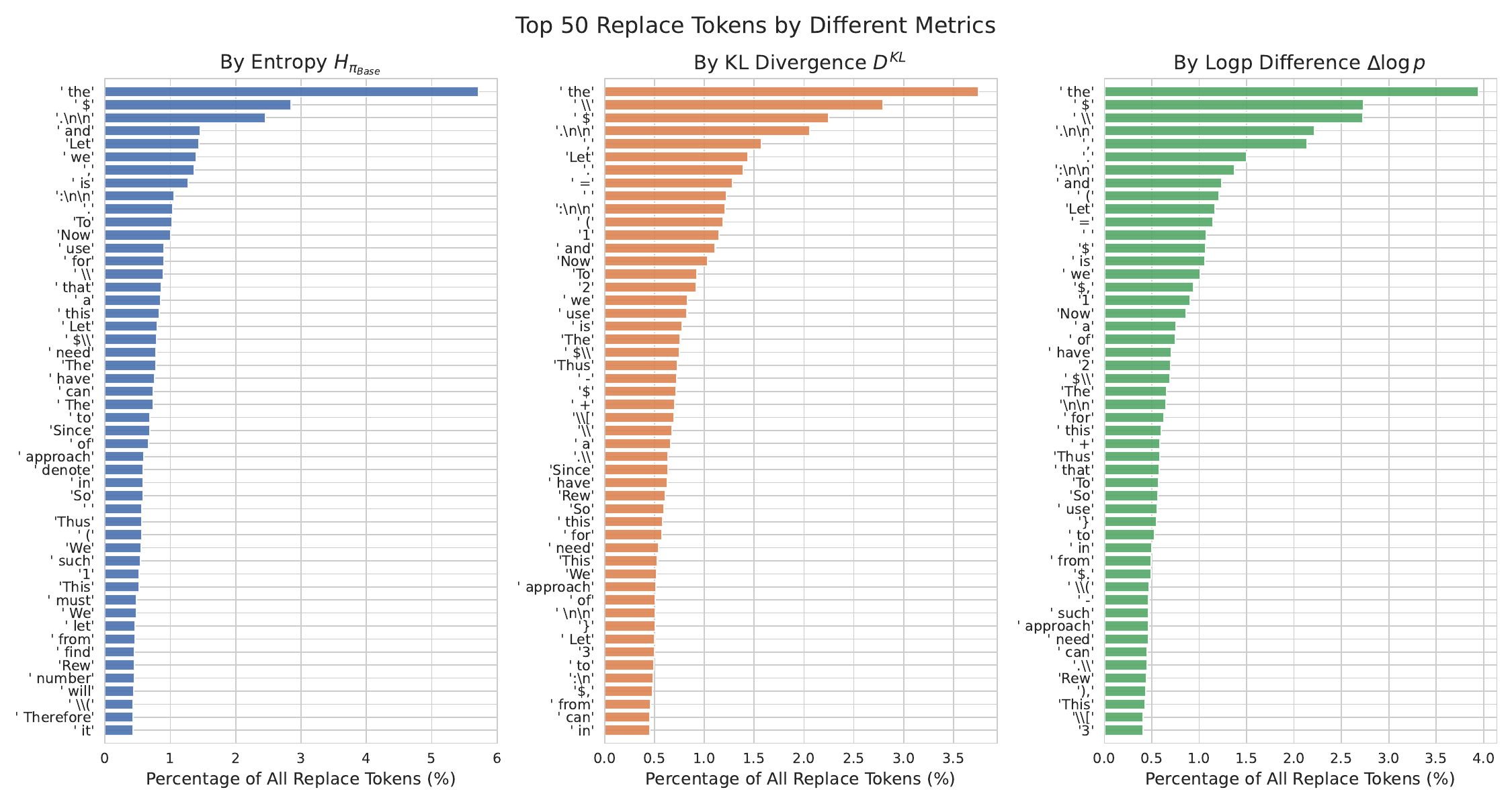}
    \caption{Top 50 tokens for replacing the base model's choice under different metrics' selection.}
    \label{fig:top_tokens}
\end{figure}

\updated{
\textbf{Per-Problem Accuracy during Replacement.}
We also report the per-problem accuracy changes in the token-replacement experiment in Fig.~\ref{fig:per_problem_acc}, to more finely examine how gradually increasing the replacement ratio affects model performance.
We observe that:
(1) There exist some problems that are inherently difficult for the model, for which the accuracy remains zero across all replacement ratios.
(2) For the remaining problems, the overall trend is that accuracy generally increases as the replacement ratio grows, and then begins to fluctuate. This is consistent with the fact that, when only performing token replacement, the performance is ultimately capped by the upper bound of the RLVR model.
(3) For a small number of problems, accuracy initially drops when we introduce a small amount of replacement, and then begins to improve as the replacement ratio continues to increase (\eg problem 0 of DAPO).
A qualitative inspection of these cases suggests that, for some of them, a small number of RL-replaced tokens introduce token options that the base model is not familiar with. 
As a result, the base model fails to continue the generation coherently, leading to an initial degradation in accuracy. 
However, when we further increase the replacement ratio, the generation becomes more strongly guided by the RL tokens, and the model’s performance on these problems recovers and improves.
}
\begin{figure}[ht]
    \centering
    \begin{subfigure}[c]{\textwidth}
        \includegraphics[width=\linewidth]{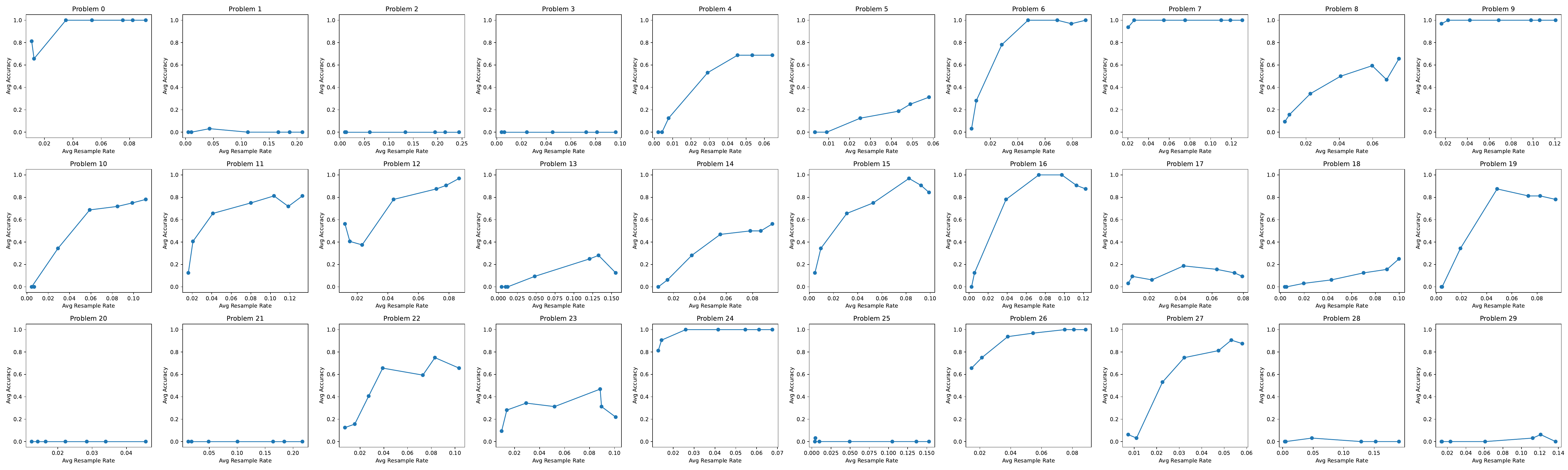}
        \caption{Per-problem accuracy on AIME24 of DAPO's token replacement experiment}
    \end{subfigure}

    \begin{subfigure}[c]{\textwidth}
        \includegraphics[width=\linewidth]{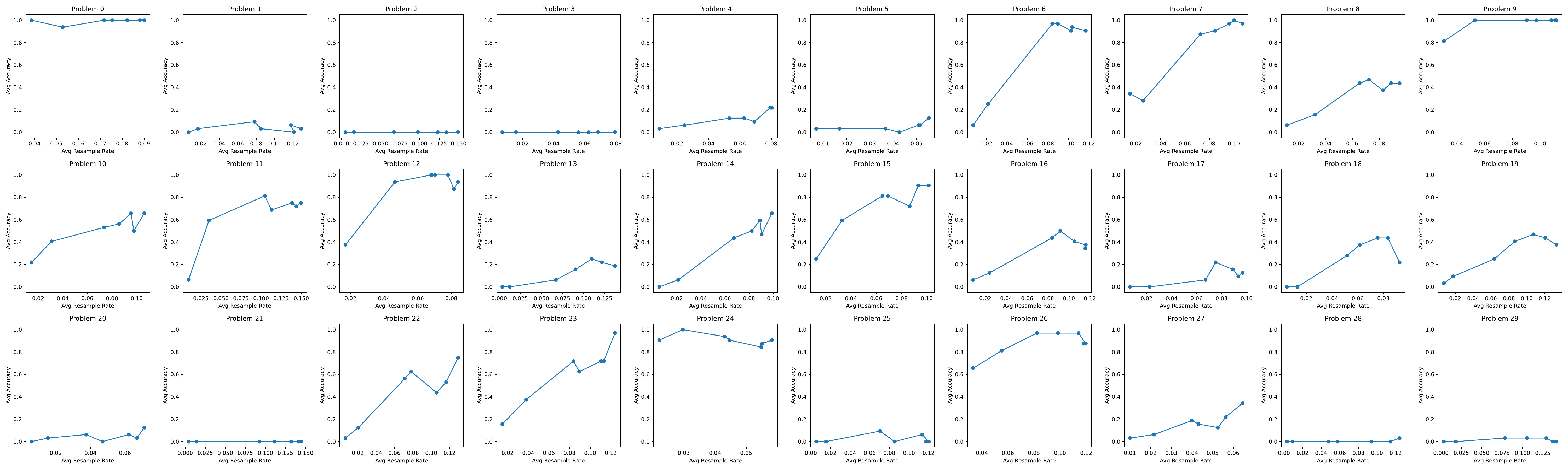}
        \caption{Per-problem accuracy on AIME24 of ORZ's token replacement experiment}
    \end{subfigure}

    \begin{subfigure}[c]{\textwidth}
        \includegraphics[width=\linewidth]{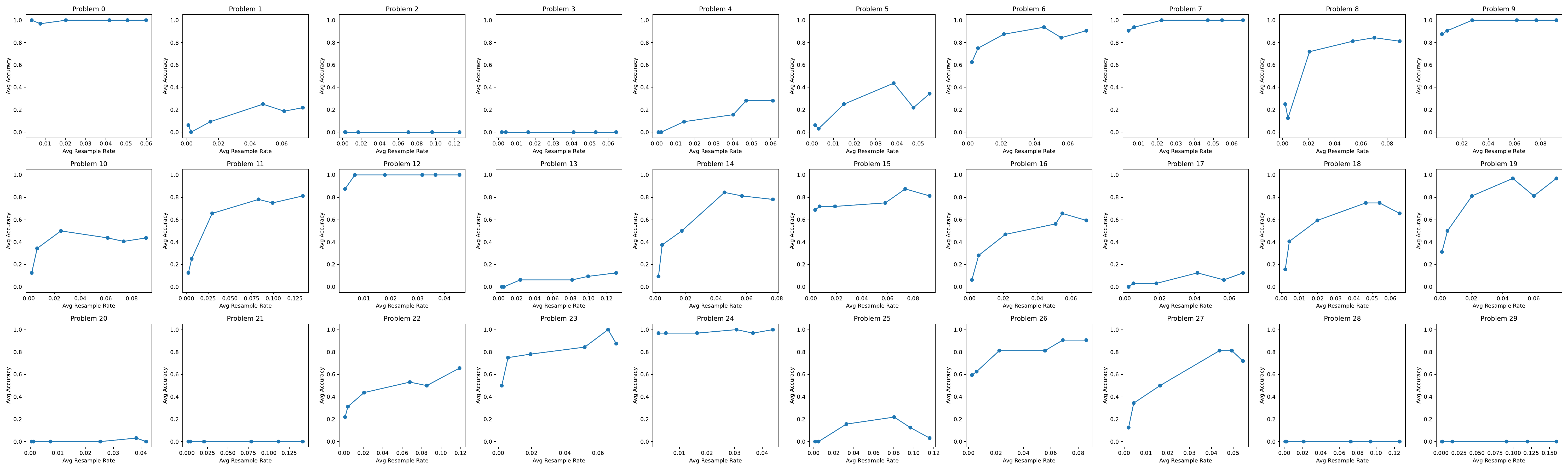}
        \caption{Per-problem accuracy on AIME24 of UniReason's token replacement experiment}
    \end{subfigure}

    \caption{\updated{Per-problem accuracy changes on AIME24 during each model's selective token replacement experiment. We report the results with $\Delta\log p$ being the selection criterion.}}
    \label{fig:per_problem_acc}
\end{figure}

\subsection{\updated{Hyperparameter Sensitivity Analysis}}

\updated{
Our test-time extrapolation distribution $\pi_\mathrm{Extra}^\gamma$ introduces a hyperparameter $\gamma$ that determines the strength of extrapolation along the learned $\Delta\log p$ direction.
This intervention operates within the token replacement procedure (Algo.~\ref{alg:replace}) and is applied only to tokens selected by the criterion $\Delta\log p < \tau$.
To verify the robustness of the performance gain of extrapolation over simply replacing the token from $\pi_\mathrm{RL}$, we perform a grid search over both $\gamma$ and the token-selection threshold $\tau$.
We evaluate $\gamma\in\{0.05, 0.1\}$ and vary $\tau$ across different ranges for different models.
For DAPO and ORZ, we test $\tau\in\{-0.5, -0.4, -0.3, -0.2, -0.1\}$.
For UniReason, we adopt a denser grid $\tau\in\{-0.5, -0.45, -0.4, -0.35, -0.3\}$ because relatively few replacements are needed to reach the RLVR performance level (Fig.~\ref{fig:replace}).
}

\updated{
As shown in Tab.~\ref{tab:full_param_extra}, across nearly all models and hyperparameter settings, extrapolation consistently outperforms the replace-only variant, demonstrating a strong robustness of our method. 
Notably, once the replacement ratio is sufficiently high to match the RLVR's performance, further increases in replacement provide little to no additional benefit, since the performance is bounded by the RLVR model itself. 
In contrast, a proper test-time extrapolation can further exceed RLVR performance by 1–3 points without any additional training.
}
\begin{table}[ht]
\centering
\caption{\updated{Hyperparameter sensitivity analysis for the selective extrapolation experiment. The $^*$sign marks the reported value for extrapolation results in Fig.~\ref{fig:extrapolate}, while the $^\dagger$sign corresponds to the end point in token replacement of Fig.~\ref{fig:replace}.}}
\label{tab:full_param_extra}
\begin{subtable}{\textwidth}
\centering
\caption{Hyperparameters and Avg@32 performance on AIME24 of DAPO (Avg@32 of $\pi_\mathrm{RL}$: 52.60).}
\begin{tabular}{@{}cccccc@{}}
\toprule
\textbf{Threshold $\tau$} & -0.5 & -0.4 & -0.3 & -0.2 & -0.1 \\ 
\textit{Average Replace ratio} & 8.8\% & 10.0\% & 11.4\% & 13.4\% & 16.5\% \\ \midrule
\textbf{Replace w/ $\pi_\mathrm{RL}$} & {\ul 51.98$^\dagger$} & 51.56 & 51.67 & 52.71 & 51.98 \\
\textbf{Extrapolate w/ $\gamma = 0.05$} & 51.88 & {\ul 53.02} & \textbf{55.42$^*$} & \textbf{54.06} & \textbf{54.9} \\
\textbf{Extrapolate w/ $\gamma = 0.1$} & \textbf{54.17} & \textbf{53.33} & {\ul 53.85} & {\ul 53.85} & {\ul 54.27} \\ \bottomrule
\end{tabular}
\end{subtable}

\vspace{1em}

\begin{subtable}{\textwidth}
\centering
\caption{Hyperparameters and Avg@32 performance on AIME24 of ORZ (Avg@32 of $\pi_\mathrm{RL}$: 46.15).}
\begin{tabular}{@{}cccccc@{}}
\toprule
\textbf{Threshold $\tau$} & -0.5 & -0.4 & -0.3 & -0.2 & -0.1 \\
\textit{Average Replace ratio} & 9.5\% & 10.1\% & 10.8\% & 11.6\% & 12.7\% \\ \midrule
\textbf{Replace w/ $\pi_\mathrm{RL}$} & 43.65 & 43.33 & \textbf{46.15$^\dagger$} & 44.90 & 42.81 \\
\textbf{Extrapolate w/ $\gamma = 0.05$} & \textbf{47.19} & {\ul 45.52} & {\ul 45.83} & {\ul 46.25} & {\ul 43.44} \\
\textbf{Extrapolate w/ $\gamma = 0.1$} & {\ul 43.75} & \textbf{47.50$^*$} & 45.52 & \textbf{47.08} & \textbf{45.42} \\ \bottomrule
\end{tabular}
\end{subtable}

\vspace{1em}

\begin{subtable}{\textwidth}
\centering
\caption{Hyperparameters and Avg@32 performance on AIME24 of UniReason (Avg@32 of $\pi_\mathrm{RL}$: 54.58).}
\begin{tabular}{@{}cccccc@{}}
\toprule
\textbf{Threshold $\tau$} & -0.5 & -0.45 & -0.4 & -0.35 & -0.3 \\
\textit{Average Replace ratio} & 5.4\% & 6.0\% & 6.8\% & 7.5\% & 8.5\% \\ \midrule
\textbf{Replace w/ $\pi_\mathrm{RL}$} & {\ul 53.65$^\dagger$} & 53.33 & 53.12 & 54.06 & 53.54 \\
\textbf{Extrapolate w/ $\gamma = 0.05$} & 51.88 & \textbf{54.79} & {\ul 53.54} & {\ul 55.00} & {\ul 54.69} \\
\textbf{Extrapolate w/ $\gamma = 0.1$} & \textbf{54.37} & {\ul 53.75} & \textbf{53.96} & \textbf{55.83$^*$} & \textbf{55.10} \\ \bottomrule
\end{tabular}
\end{subtable}
\end{table}

\section{RLVR Training Details}

\textbf{Hyperparameters Setting.}
We adopt the open-sourced \href{https://github.com/verl-project/verl/tree/v0.5.0/recipe/dapo}{DAPO recipe} for RLVR training.
Our configuration includes double clip ratios ($\epsilon_\mathrm{low}=0.2$ and $\epsilon_\mathrm{high}=0.28$) and a learning rate of 1e-6 with a 10-step warmup.
Each RLVR step consists of 512 prompts with 16 sampled responses each, processed in mini-batches of 32 prompts to yield 16 gradient updates per step. 
Maximum generation length (and overlong penalty thresholds) are set to 8k (4k) for Qwen2.5-Math-7B and 20k (16k) for Qwen3-8b-base, respectively.

For reweighting, our parameter $\alpha$ (Eq.~\ref{eqn:reweight}) is set to 0.2 for Qwen2.5 and 0.1 for Qwen3. 
Following the recommended values by \citet{decomposing_zhaoxin} and \citet{over_dominate}, we set $\alpha$ to $0.1$ for $\tilde A_{i,t}^\mathrm{dom}$ and $0.01$ for $\tilde A_{i,t}^\mathrm{PPL}$.
For $\tilde A_{i,t}^\mathrm{dom}$ specifically, we also adjust $\epsilon_\mathrm{high}$ to $0.24$.

\begin{figure}[htbp]
    \centering
    \includegraphics[width=\linewidth]{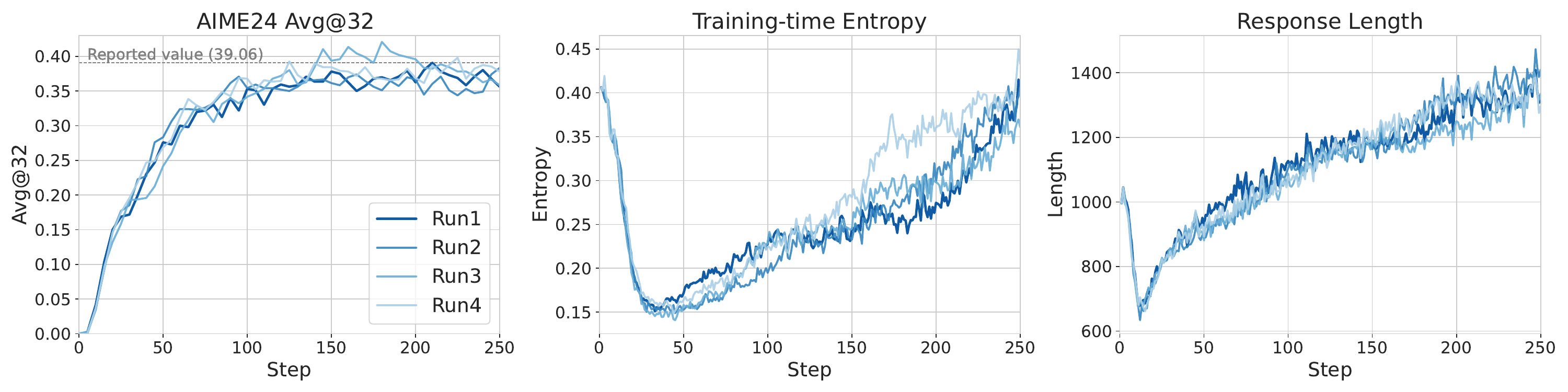}
    \caption{Reproducibility analysis. 
    The learning curves across 4 independent runs on Qwen2.5-Math-7B with our reweighting method (Eq.~\ref{eqn:reweight}) show consistent convergence and performance.}
    \label{fig:replicate}
\end{figure}

\textbf{Reproducibility Analysis.}
To account for random variations in the RL process, we also performed four separate training runs on the Qwen2.5-Math-7B backbone for our reweighting method. 
Fig.~\ref{fig:replicate} displays the learning curves for these experiments (Run 1--4, where Run 1 is our reported run in Fig.~\ref{fig:training_curve}). 
The results indicate that our method is highly reproducible; across all trials, the model reached or surpassed the performance levels presented in Tab.~\ref{tab:compare}.

\section{\updated{Performance beyond Pure-Math Reasoning Tasks}}

\updated{
Although our models are primarily trained and evaluated on math-focused datasets, it is important to assess their reasoning ability on non-math tasks to evaluate generalization ability.
Following prior work \citep{GMPO}, we use the Minerva dataset \citep{minerva}, which contains 272 undergraduate-level STEM problems spanning diverse subjects such as Chemistry and Astronomy\footnote{The dataset is named as \texttt{OCWCourses} in the paper, which can be found in \url{https://openreview.net/attachment?id=IFXTZERXdM7&name=supplementary_material}.}.
}

\updated{
We begin by benchmarking the RLVR-trained models on Minerva using the same sampling parameters as in other evaluations (\eg AIME24).
As shown in Tab.~\ref{tab:rl_minerva}, models trained with our reweighting method continue to outperform baselines in reasoning accuracy. Importantly, these gains do not come at the expense of exploration ability, as reflected by comparable or improved Pass@k scores.
}

\updated{
We further evaluate test-time extrapolation on Minerva. 
Because Minerva is substantially larger than AIME24 (around 7 times more questions), we report Avg@8 for the evaluated 14B–32B models.
As shown in Fig.~\ref{fig:extra_minerva}, test-time extrapolation consistently improves over the RLVR model’s accuracy, validating its generalization ability beyond pure-math datasets.
We also report the hyperparameter grids in Tab.~\ref{tab:minerva_param_extra}, where the extrapolation results also consistently outperform replacing with $\pi_\mathrm{RL}$ only.
}

\begin{table}[ht]
\centering
\caption{\updated{Performance of RLVR-trained models on Minerva.}}
\label{tab:rl_minerva}
\begin{subtable}{\textwidth}
\centering
\caption{On Qwen2.5-Math-7B}
\begin{tabular}{@{}cccccc@{}}
\toprule
\textbf{Method} & \textbf{Base} & \textbf{DAPO} & \textbf{PPL} & \textbf{Dominate} & \textbf{Ours} \\ \midrule
\textbf{Avg@32} & 18.35 & 46.43 & {\ul 48.68} & 47.01 & \textbf{49.72} \\
\textbf{Pass@16} & 61.04 & {\ul 69.44} & 68.69 & 64.59 & \textbf{70.37} \\ \bottomrule
\end{tabular}
\end{subtable}

\vspace{1em}

\begin{subtable}{\textwidth}
\centering
\caption{On Qwen3-8B-Base}
\begin{tabular}{@{}cccc@{}}
\toprule
\textbf{Method} & \textbf{Base} & \textbf{DAPO} & \textbf{Ours} \\ \midrule
\textbf{Avg@32} & 29.8 & {\ul 55.04} & \textbf{56.57} \\
\textbf{Pass@16} & 70.43 & \textbf{76.98} & {\ul 76.78} \\ \bottomrule
\end{tabular}
\end{subtable}
\end{table}

\begin{figure}[ht]
    \centering
    \includegraphics[width=0.5\linewidth]{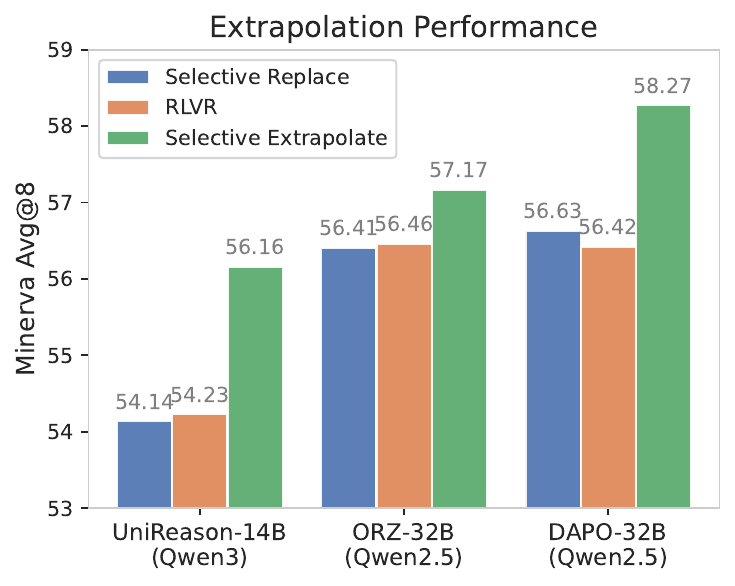}
    \caption{\updated{Extrapolation results on Minerva.}}
    \label{fig:extra_minerva}
\end{figure}

\begin{table}[ht]
\centering
\caption{\updated{Hyperparameters and Avg@8 performance on Minerva benchmark. The $^*$sign marks the tuned value in Fig.~\ref{fig:extra_minerva}.}}
\label{tab:minerva_param_extra}
\begin{tabular}{@{}c|cc|cc|cc@{}}
\toprule
 & \multicolumn{2}{c|}{DAPO} & \multicolumn{2}{c|}{ORZ} & \multicolumn{2}{c}{UniReason} \\ \midrule
\textbf{Threshold $\tau$} & -1.0 & -0.9 & -1.0 & -0.9 & -1.0 & -0.9 \\
\textit{Avg replace ratio} & 6.5\% & 7.0\% & 9.2\% & 9.6\% & 1.8\% & 2.2\% \\ \midrule
\textbf{Replace w/ $\pi_\mathrm{RL}$} & 56.63 & 56.43 & 56.41 & 56.39 & 54.00 & 54.14 \\
\textbf{Extrapolate w/ $\gamma=0.05$} & 56.8 & \textbf{57.22} & \textbf{57.17$^*$} & \textbf{57.08} & \textbf{54.50} & 54.50 \\
\textbf{Extrapolate w/ $\gamma=0.1$} & \textbf{58.27$^*$} & 56.57 & 55.51 & 55.28 & 54.32 & \textbf{56.16$^*$} \\ \bottomrule
\end{tabular}
\end{table}

\section{Proofs}

\begin{proof}[Proof of Lemma~\ref{lemma:grad}]\label{proof:grad}
For ease of notation, we omit the context $x,y_{i,<t}$ here.
The derivative of DAPO on an unclipped token $y_{i,t}$ is:
$$
\begin{aligned}
\nabla_\theta \mathcal J_\mathrm{DAPO}(y_{i,y}) = \nabla_\theta\  r_{i,t}(\theta)\hat A_{i,t} 
&= \nabla_\theta\ \frac{\pi_\theta(y_{i,t})}{\pi_{\theta_\mathrm{old}}(y_{i,t})}\hat A_{i,t} \\
&= r_{i,t}(\theta)\hat A_{i,t}\cdot\nabla_\theta\log\pi_\theta(y_{i,t}) \\
&= w_{i,t}\cdot\nabla_\theta\log\pi_\theta(y_{i,t}).
\end{aligned}
$$
For the softmax-parameterized policy $\pi_\theta$ with logits $z$ for $y_{i,t}$, assuming $y_{i,t}$ corresponds to index $k$ of vocabulary $\mathcal V$, we have:
$$
\begin{aligned}
\frac{\partial}{\partial z_j} \log\pi_\theta(y_{i,t}) &= \frac1{\pi_\theta(y_{i,t})} \cdot \frac{\partial}{\partial z_j}\frac{\exp(z_k)}{\sum_{l}\exp(z_l)} \\
&= \frac1{\pi_\theta(y_{i,t})} \cdot \left\{
\begin{array}{ll}
\frac{\exp(z_k)\sum_{l}\exp(z_l) - \exp(z_k)\exp(z_k)}{(\sum_{l}\exp(z_l))^2}, & j=k \\
\frac{-\exp(z_k)\exp(z_j)}{(\sum_{l}\exp(z_l))^2}, & j\neq k
\end{array}
\right.\\
&= \left\{
\begin{array}{ll}
1 - \pi_\theta(\mathcal V_k), & j=k \\
-\pi_\theta(\mathcal V_j), & j\neq k
\end{array}
\right. \\
&= \mathbb I(j=k) - \pi_\theta(\mathcal V_j).
\end{aligned}
$$
So the $\ell1$-norm of $\nabla_z \mathcal J_\mathrm{DAPO}(y_{i,t})$ becomes:
$$
\begin{aligned}
\left\Vert\nabla_z \mathcal J_\mathrm{DAPO}(y_{i,t})\right\Vert_1 &= \left\Vert w_{i,t}\nabla_z\log\pi_\theta(y_{i,t})\right\Vert_1\\
&= |w_{i,t}|\cdot \sum_j\Big|\mathbb I(j=k) - \pi_\theta(\mathcal V_j)\Big| \\
&= |w_{i,t}|\cdot \Big(1-\pi_\theta(y_{i,t}) + \sum_{j\neq k} \pi_\theta(\mathcal V_j)\Big) \quad (y_{i,t} = \mathcal V_k)\\
&= |w_{i,t}|\cdot 2\big(1-\pi_\theta(y_{i,t})\big).
\end{aligned}
$$
\end{proof}

\begin{proof}[Proof of Theorem~\ref{theorem:extra}]\label{proof:extra}
Let $\mathcal J(\theta_x) = \mathbb E_{y\sim\pi_{\theta_x}(\cdot)}[R_{x,y}]$, and we need to show that for each $x$:
$$\exists\ \gamma>0, \mathcal J(\theta^t_x+\gamma(\theta^t_x-\theta^0_x)) \ge \mathcal J(\theta^t_x).$$
Denote the extrapolation direction as $d^t_x = \theta^t_x - \theta^0_x$, this is equivalent to showing the directional derivative of $\mathcal J$ at $\theta^t_x$ along $d^t_x$ is positive.

The directional derivative is given by:
$$
\nabla_{d_x^t}\mathcal J(\theta^t) = \nabla_{\theta_x}\mathcal J(\theta_x^t)^\top \frac{d_x^t}{\Vert d_x^t\Vert} = \frac{1}{\Vert d_x^t\Vert}\cdot\sum_y \frac{\partial\mathcal J(\theta_x^t)}{\partial \theta_{x,y}} d_{x,y}^t.
$$
For the softmax policy $\pi_{\theta_x}(y) = \exp(\theta_{x,y})/\sum_{y'}\exp(\theta_{x,y'})$, its gradient satisfies:
$$
\frac{\partial \pi_{\theta_x}(y')}{\partial \theta_{x,y}} = \pi_{\theta_x}(y')\left(\mathbb I(y=y') - \pi_{\theta_x}(y)\right).
$$
So the partial gradient of $\mathcal J$ on $y$ is:
$$
\frac{\partial\mathcal J(\theta_x)}{\partial \theta_{x,y}} = \sum_{y'}R_{x,y'}\frac{\partial \pi_{\theta_x}(y')}{\partial \theta_{x,y}} = R_{x,y}\pi_{\theta_x}(y) - \pi_{\theta_x}(y)\sum_{y'}R_{x,y'}\pi_{\theta_x}(y') = \pi_{\theta_x}(y)(R_{x,y} - \pi_{\theta_x}^\top R_x).
$$
Note that the advantage is $A^t(x,y) = R_{x,y} - \pi_{\theta_x^t}^\top R_x$ under the bandit setting, the directional derivative thus becomes:
$$
\begin{aligned}
\nabla_{d_x^t}\mathcal J(\theta^t) &= \frac{1}{\Vert d_x^t\Vert}\cdot\sum_y \pi_{\theta_x^t}(y)(R_{x,y} - \pi_{\theta_x^t}^\top R_x)d_{x,y}^t \\
&= \frac{1}{\Vert d_x^t\Vert}\cdot\sum_a \pi_{\theta_x^t}(y)\cdot A^t(x,y)\cdot d_{x,y}^t 
\end{aligned}
$$

We now analyze the order of $A^t(x,y)$ and $d_{x,y}^t$.

Under the assumed bandit setting, the order of $A^t(x,y)$ is the same as the order of $R_{x,y}$, i.e., $A^t(x,y_1) > A^t(x,y_2)$ if and only if $R_{x,y_1} > R_{x,y_2}$.
For $d_{x,y}^t$, we can prove that its order is also the same as $R_{x,y}$ with induction.

At $t=1$, using the update rule of NPG, we have:
$$d_{x,y}^1 - d_{x,y'}^1 = \eta\cdot (A^0(x,y) - A^0(x,y')) = \eta\cdot (R_{x,y} - R_{x,y'}).$$
So the order of $d_{x,y}^1$ is the same as $R_{x,y}$.
Assume at iteration $t$, the order of $d_{x,y}^t$ is the same as $R_{x,y}$, then at iteration $t+1$, we have:
$$d_{x,y}^{t+1} - d_{x,y'}^{t+1} = d_{x,y}^t - d_{x,y'}^t + \eta\cdot (A^t(x,y) - A^t(x,y')) = d_{x,y}^t - d_{x,y'}^t + \eta\cdot (R_{x,y} - R_{x,y'}).$$
So we still have $d_{x,y}^{t+1} > d_{x,y'}^{t+1} \iff R_{x,y} > R_{x,y'}$.
Thus by induction, the order of $d_{x,y}^t$ is the same as $R_{x,y}$ for all $t$.

Since the order of $A^t(x,y)$ and $d_{x,y}^t$ are the same, we can apply the Chebyshev sum inequality to get:
$$
\sum_y\pi_{\theta_x^t}(y) \cdot \sum_y \pi_{\theta_x^t}(y)\cdot A^t(x,y)\cdot d_{x,y}^t \ge \left(\sum_y \pi_{\theta_x^t}(y)\cdot A^t(x,y)\right)\cdot \left(\sum_y \pi_{\theta_x^t}(y)\cdot d_{x,y}^t\right),
$$
with the equality holds if and only if $A^t(x,y)$ or $d_{x,y}^t$ is a constant for all $y$ (i.e., constant reward).

Note that the expectation of advantage $\sum_y \pi_{\theta_x^t}(y)\cdot A^t(x,y) = 0$, so we have:
$$\nabla_{d_x^t}\mathcal J(\theta^t) = \frac{1}{\Vert d_x^t\Vert}\cdot\sum_y \pi_{\theta_x^t}(y)\cdot A^t(x,y)\cdot d_{x,y}^t \ge 0.$$
The equality holds if and only if $R_{x,y}$ is a constant for all $y$.

\end{proof}

\section{Statistical Comparison of Different Metrics}\label{appdix:stat}

\textbf{Empirical setup.}
We evaluate three RLVR models: ORZ, DAPO, UniReason, and their base counterparts.
For each model, we generate 32 responses per question from the AIME-24 dataset, with a sampling strategy of top-p=0.7 and temperature=1.0.
Our analysis focuses on several metrics comparing the model pairs: the base/RLVR model's entropy, KL divergences, and the logp difference.
The probability distribution versus different $\Delta\log p$ bins in Fig.~\ref{fig:probs}\textcolor{red}{\hyperref[fig:probs]{(b)}} is also measured on the DAPO's generation under this setting.

\textbf{Statistics of Different Metrics.}
We compute each metric of the three RLVR model pairs on both the base model and the RLVR model's generation.
As shown in Fig.~\ref{fig:hist_logp_diffs}, the distribution of logp difference $\Delta\log p$ is bimodal, with a positive tail for the RLVR's generated text and a negative tail for the base model's generation.
In contrast, the distributions of other magnitude-based metrics are nearly identical regardless of which model generated the output (Fig.~\ref{fig:hist_unireason_magnitudes}-\ref{fig:hist_orz_magnitudes}).

\updated{
\textbf{Word Clouds of High-$\Delta\log p$ Tokens.}
To gain qualitative insight into the tokens identified as higher $\Delta\log p$, whose probabilities are substantially increased by the RLVR training process, we generated word clouds from the top-100 high-$\Delta\log p$ tokens for each model (Figure~\ref{fig:word_cloud}).
As the figure shows, these tokens correspond to words related to problem-solving. 
They fall into two clear categories: explicit reasoning actions (\eg combine, break, simplify) and logical transitions (\eg wait, think, step). 
The prevalence of this vocabulary suggests that the RLVR model has learned to construct more effective reasoning processes.
}

\begin{figure}
    \centering
    \begin{subfigure}[c]{0.6\textwidth}
        \includegraphics[width=\linewidth]{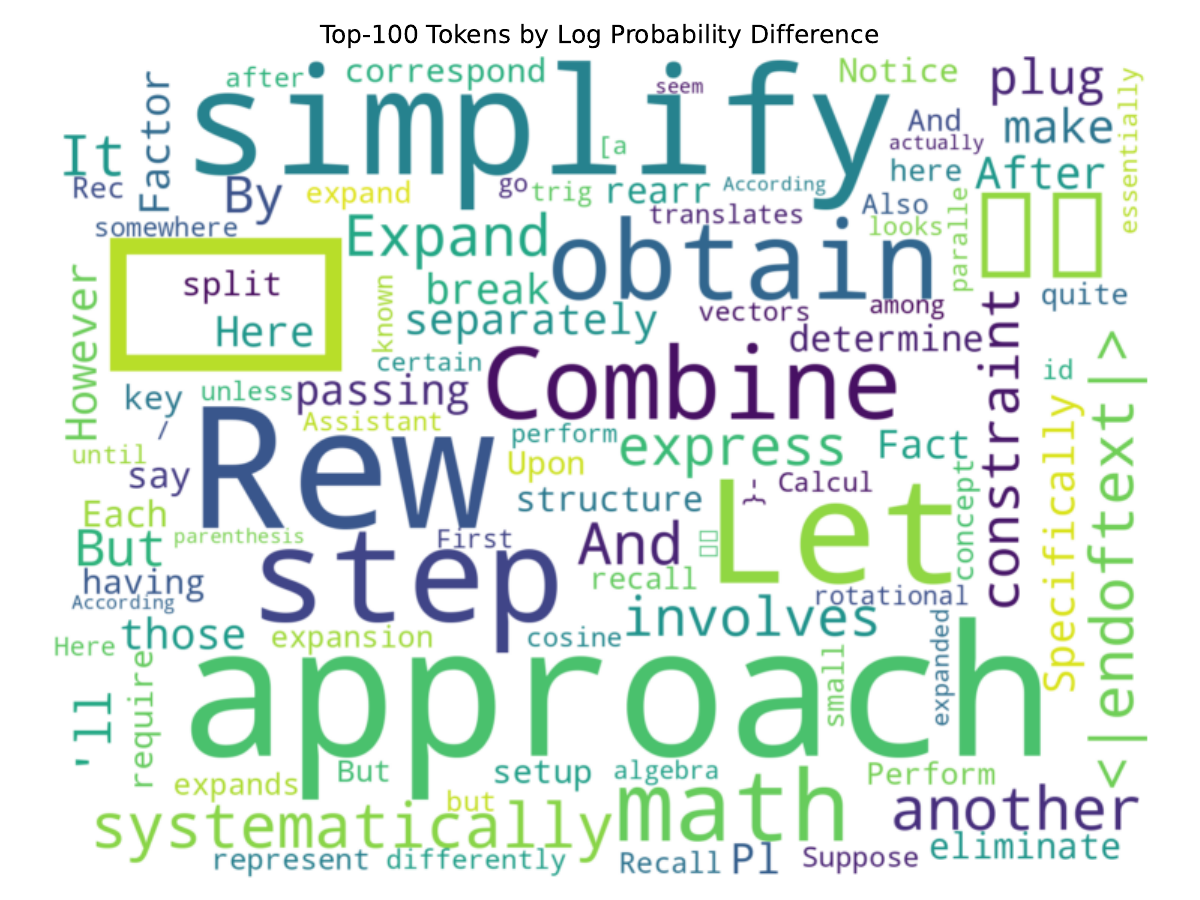}
        \caption{Top $\Delta\log p$ tokens of DAPO}
    \end{subfigure}
    
    \begin{subfigure}[c]{0.6\textwidth}
        \includegraphics[width=\linewidth]{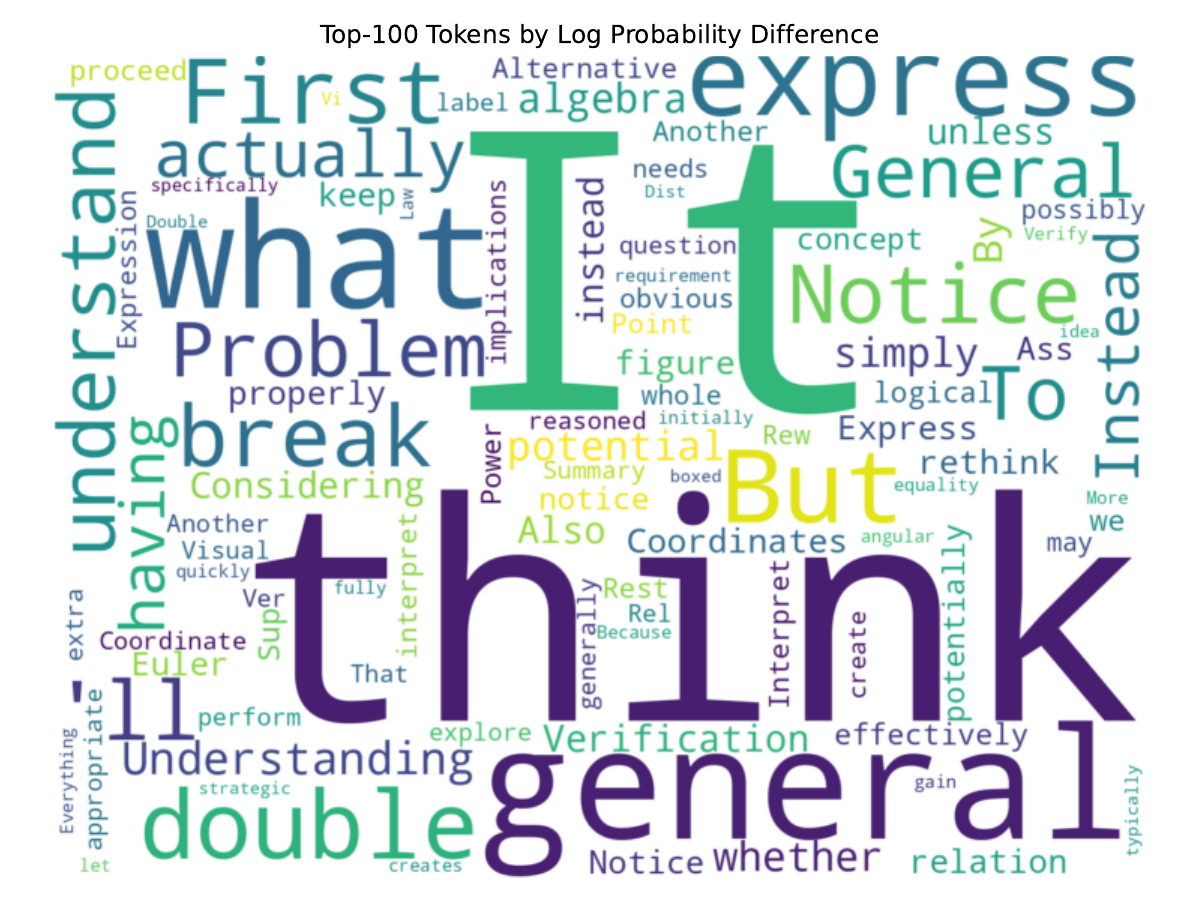}
        \caption{Top $\Delta\log p$ tokens of ORZ}
    \end{subfigure}
    
    \begin{subfigure}[c]{0.6\textwidth}
        \includegraphics[width=\linewidth]{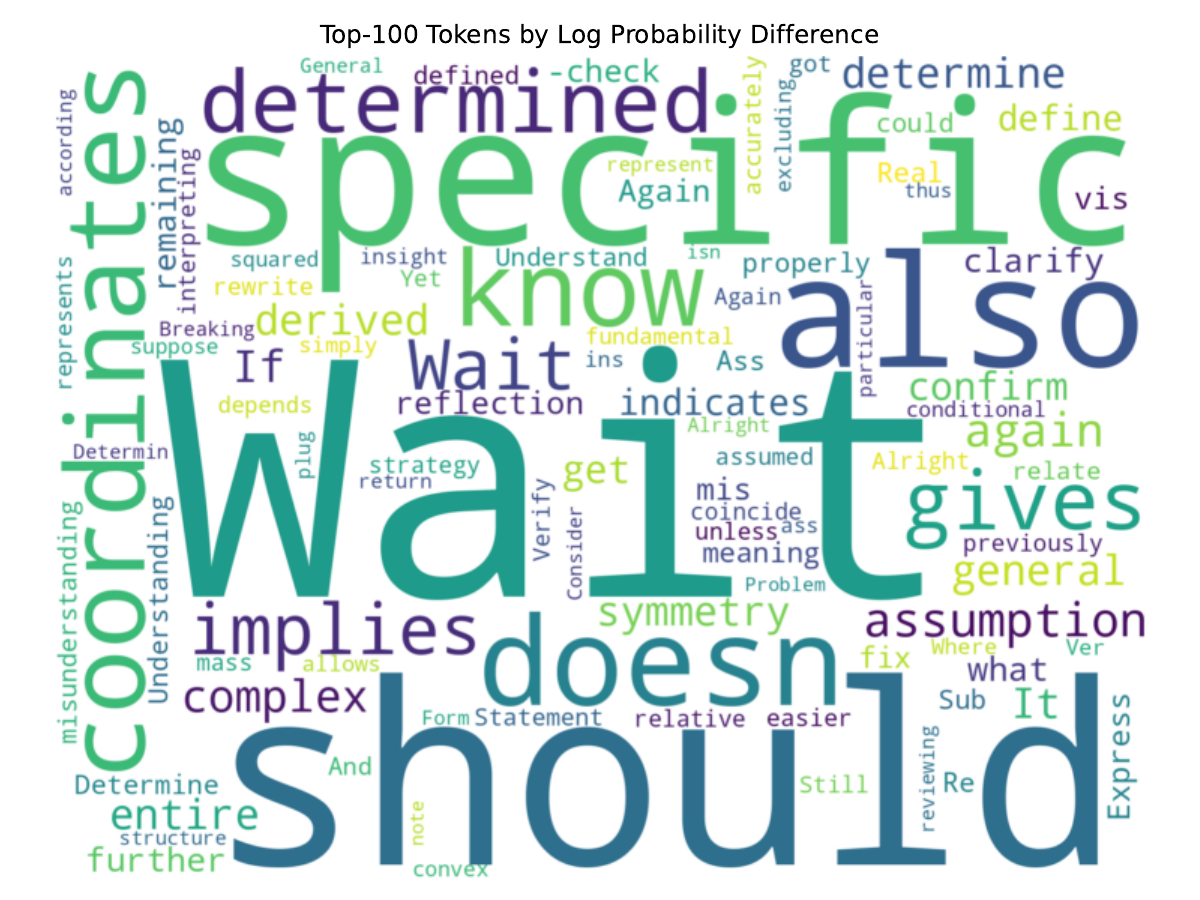}
        \caption{Top $\Delta\log p$ tokens of UniReason}
    \end{subfigure}
    \caption{Word clouds of top $\Delta\log p$ tokens, measured w/ different RLVR-trained models.}
    \label{fig:word_cloud}
\end{figure}

\section{The Use of Large Language Models}
We utilize LLMs only to polish some of the language of this paper.
All content was originally drafted by the authors. 
The use of LLMs was restricted to refining some pre-existing text, and any suggested modifications were reviewed by the authors to confirm their accuracy and alignment with the original meaning.

\begin{figure}[ht]
    \centering
    \begin{subfigure}[c]{0.8\textwidth}
        \includegraphics[width=\linewidth]{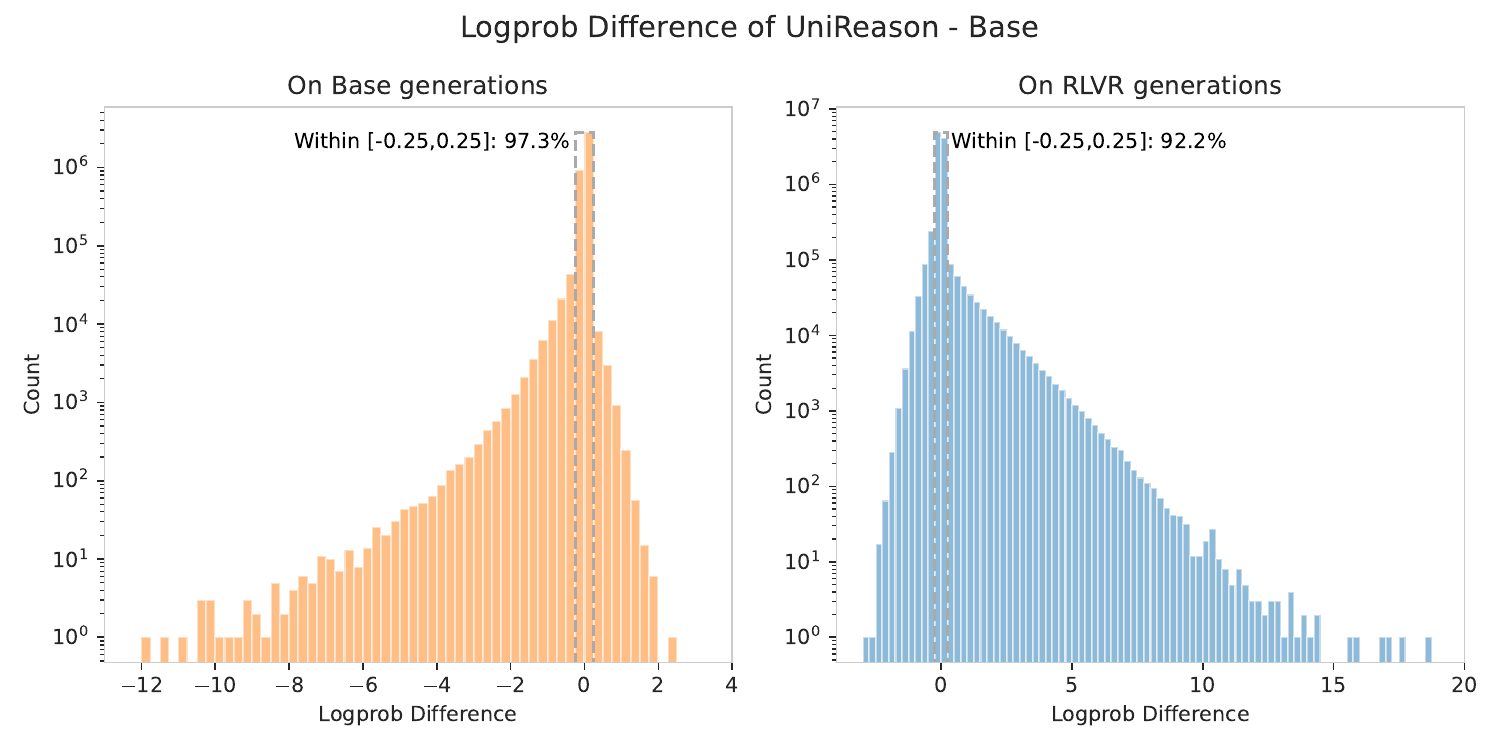}
        \caption{Logp difference of UniReason}
    \end{subfigure}

    \begin{subfigure}[c]{0.8\textwidth}
        \includegraphics[width=\linewidth]{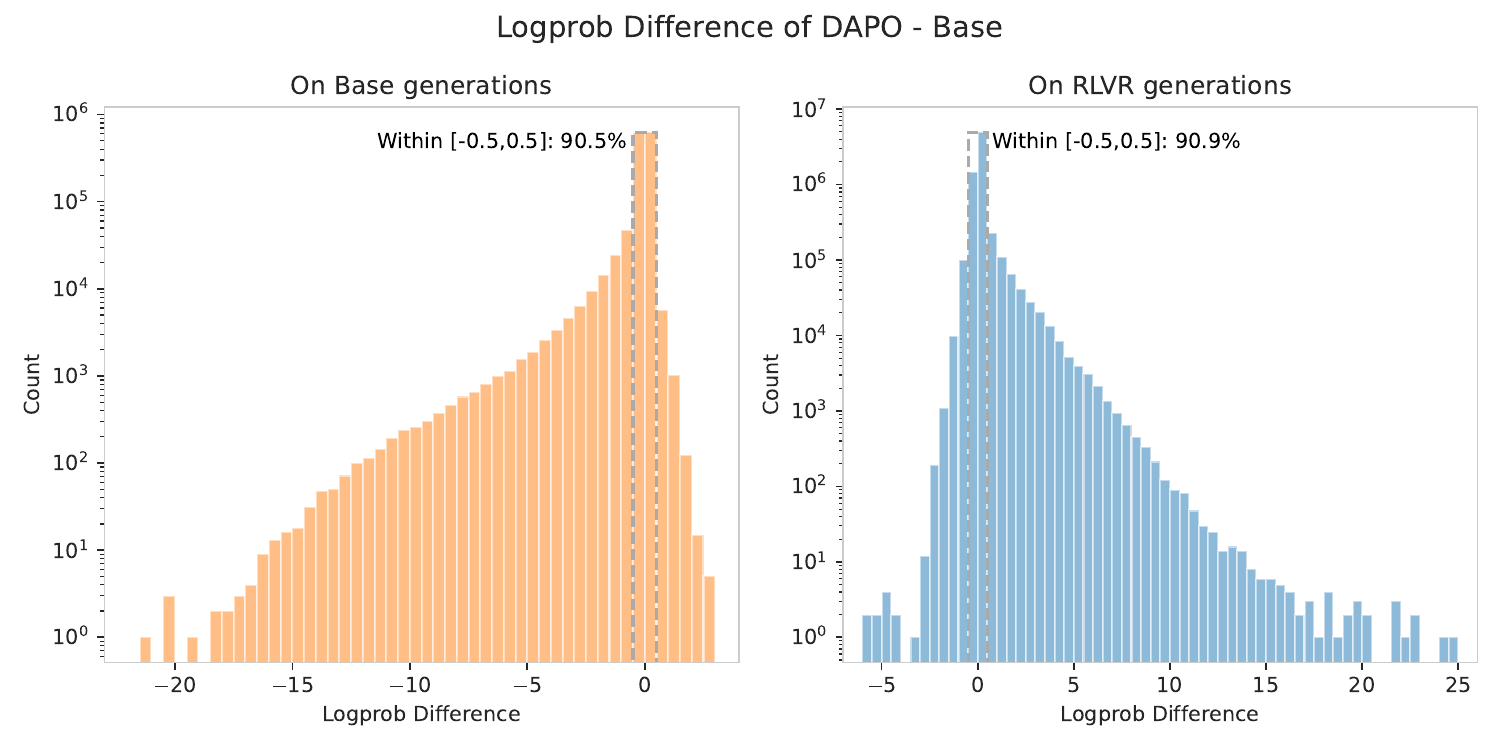}
        \caption{Logp difference of DAPO}
    \end{subfigure}

    \begin{subfigure}[c]{0.8\textwidth}
        \includegraphics[width=\linewidth]{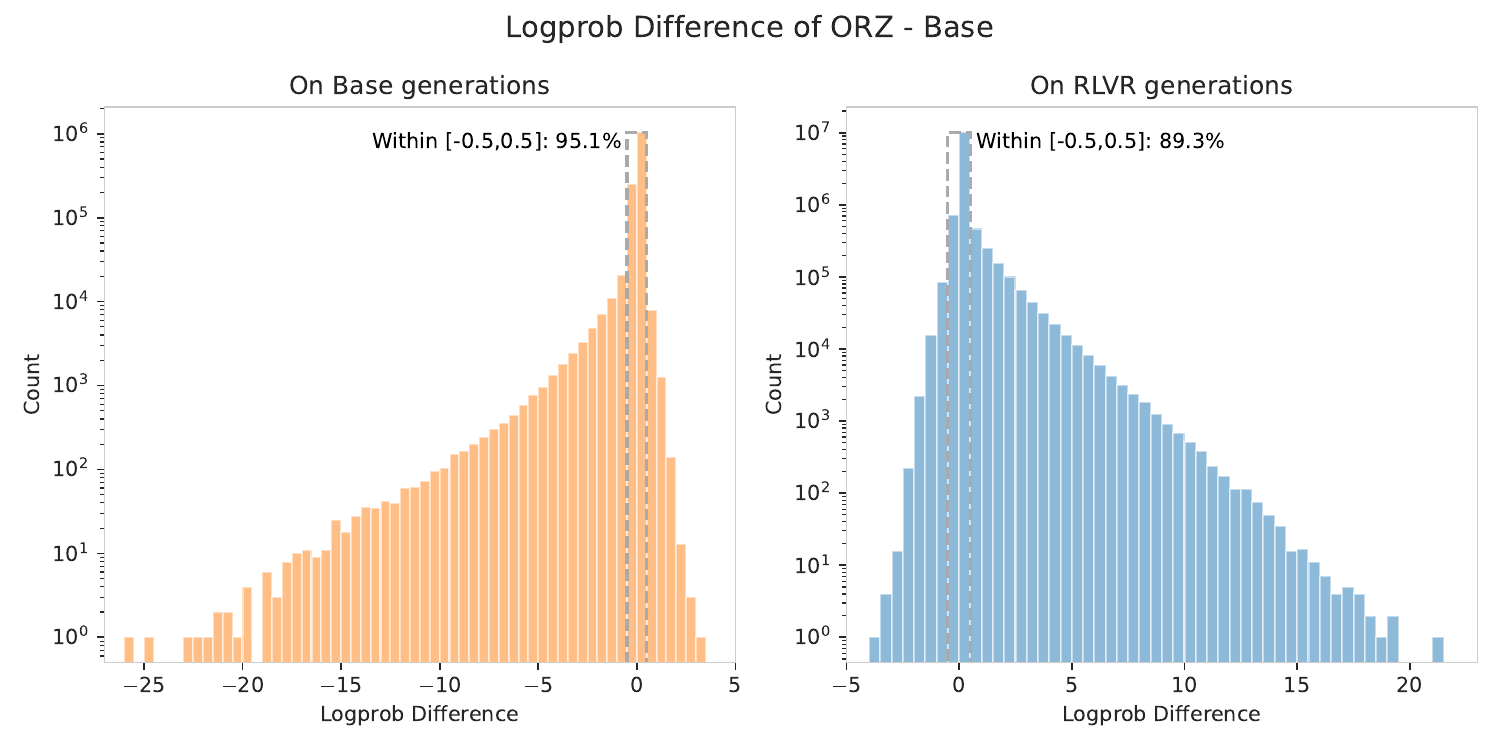}
        \caption{Logp difference of ORZ}
    \end{subfigure}

    \caption{Logp Difference histograms of different RLVR models, comparing the RLVR and base model's generations.}
    \label{fig:hist_logp_diffs}
\end{figure}

\begin{figure}[ht]
    \centering
    \begin{subfigure}[c]{0.6\textwidth}
        \centering
        \includegraphics[width=0.88\linewidth]{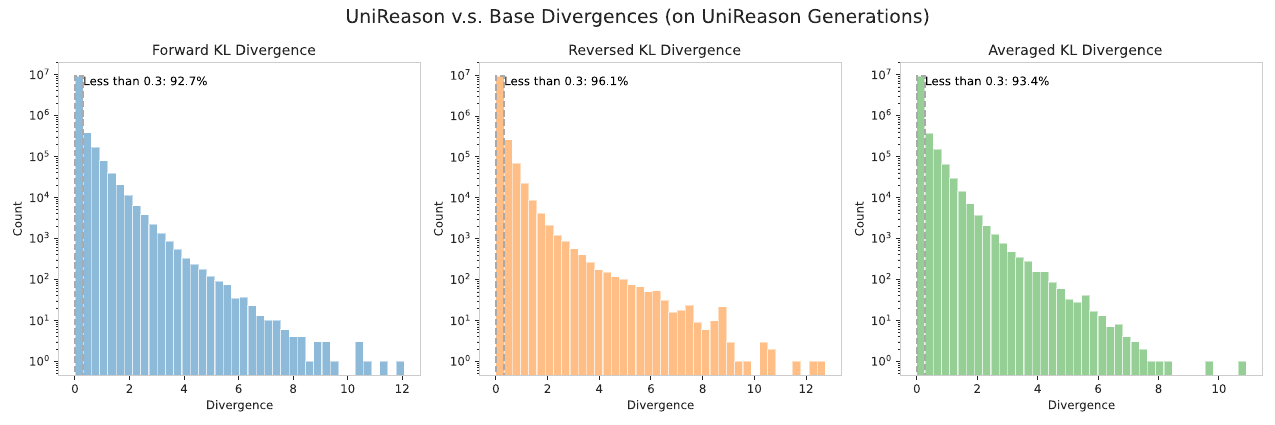}
        \caption{Divergence on UniReason's generations}
    \end{subfigure}
    \hfill
    \begin{subfigure}[c]{0.38\textwidth}
        \centering
        \includegraphics[width=0.9\linewidth]{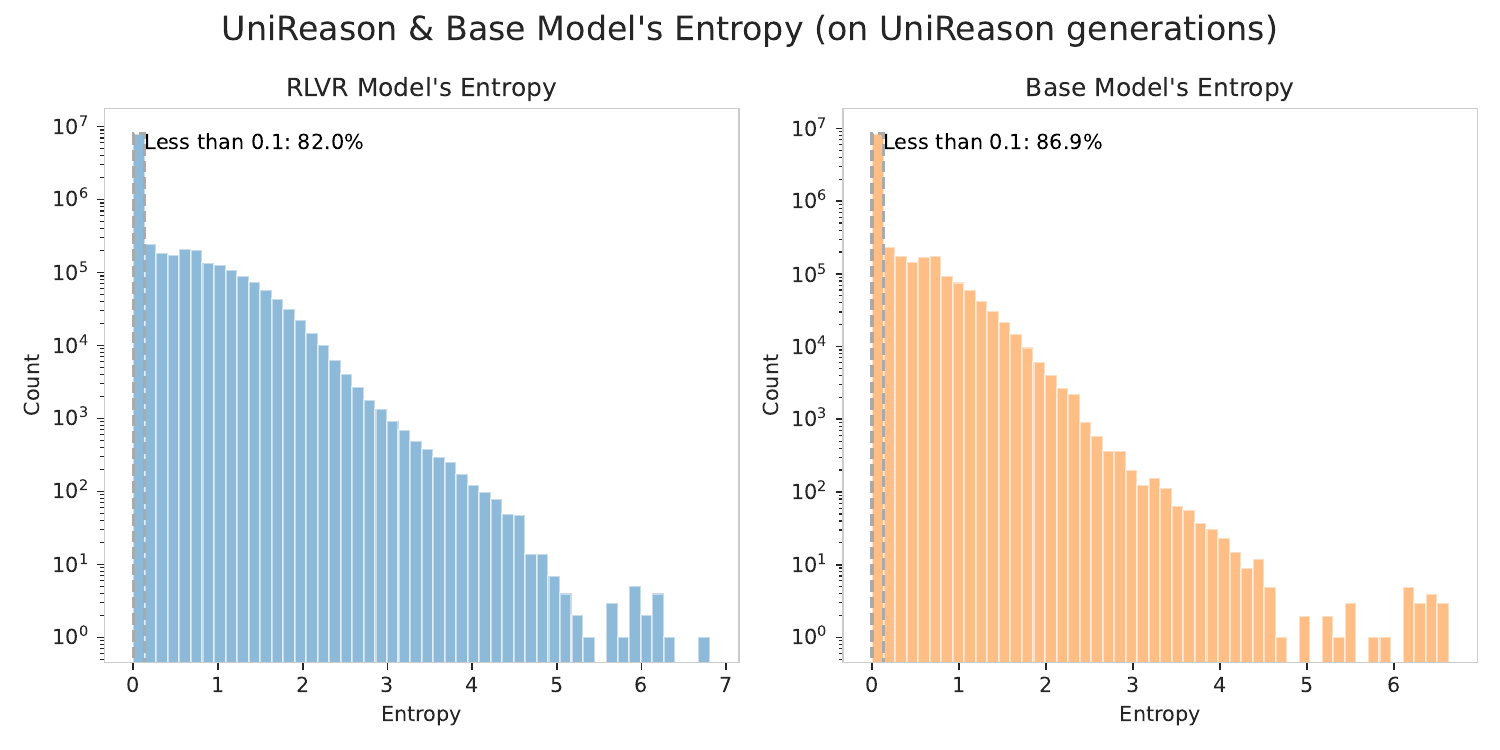}
        \caption{Entropy on UniReason's generations}
    \end{subfigure}
    
    \begin{subfigure}[c]{0.6\textwidth}
        \centering
        \includegraphics[width=0.88\linewidth]{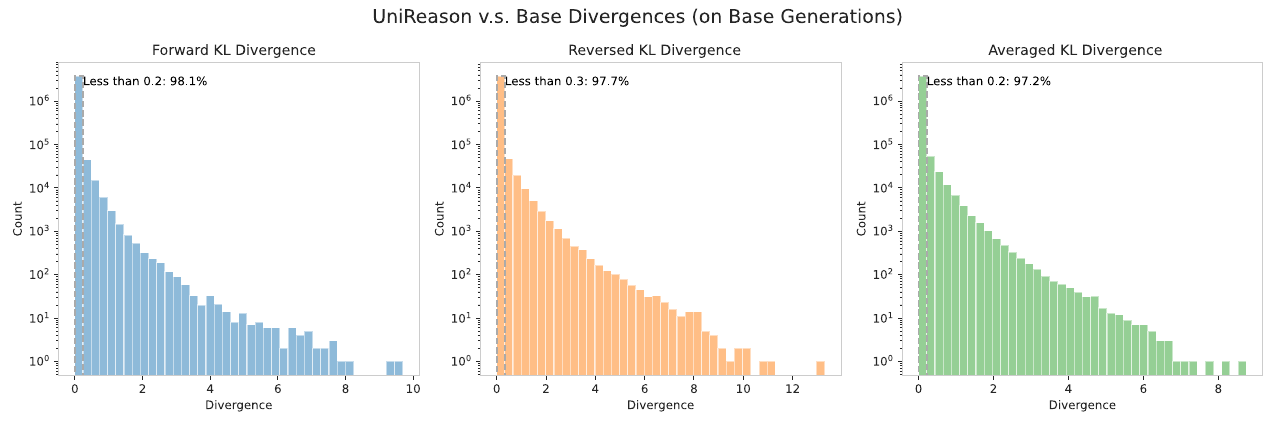}
        \caption{Divergence on base's generations}
    \end{subfigure}
    \hfill
    \begin{subfigure}[c]{0.38\textwidth}
        \centering
        \includegraphics[width=0.9\linewidth]{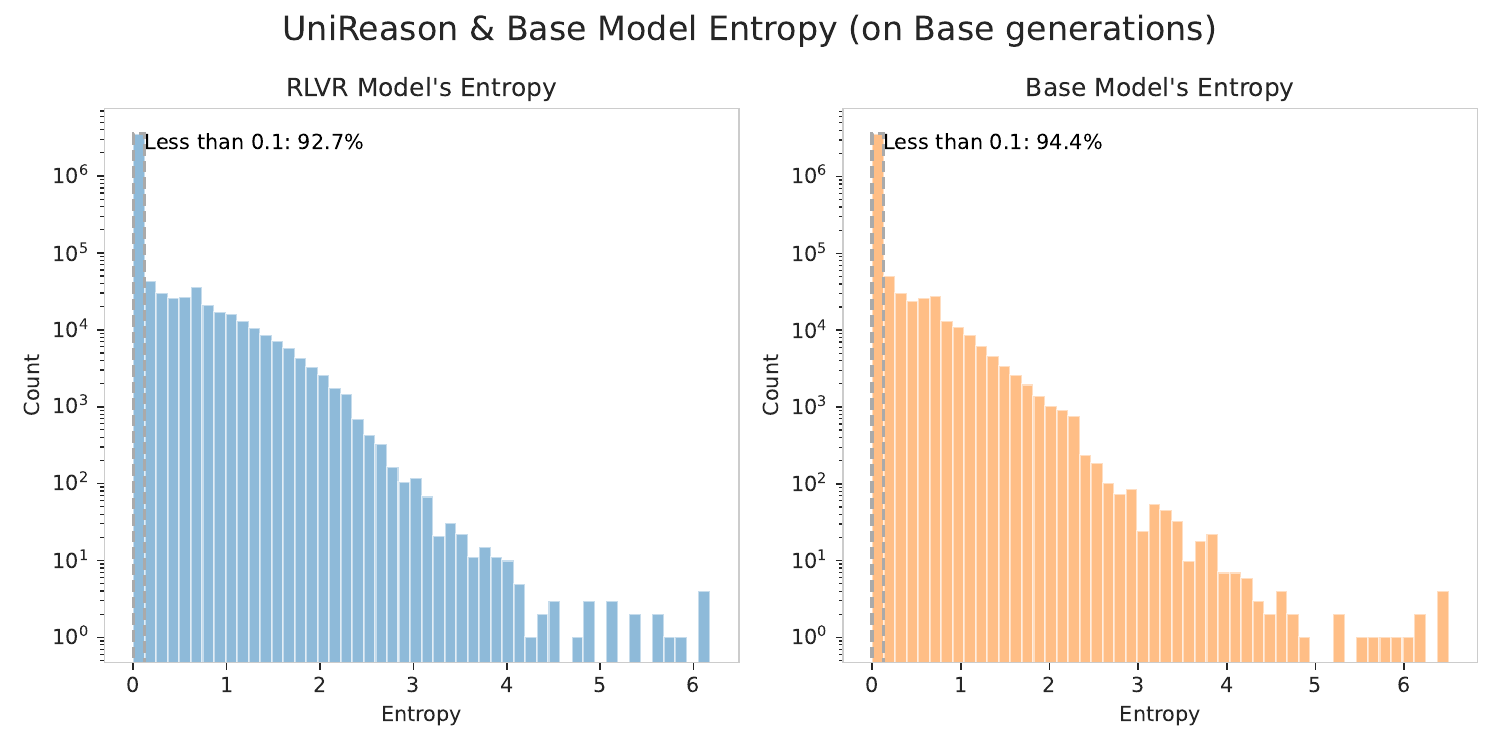}
        \caption{Entropy on base's generations}
    \end{subfigure}

    \caption{Divergence and entropy histograms of UniReason and its corresponding base model measured on UniReason or the base model's generations.}
    \label{fig:hist_unireason_magnitudes}
\end{figure}

\begin{figure}[ht]
    \centering
    \begin{subfigure}[c]{0.6\textwidth}
        \centering
        \includegraphics[width=0.88\linewidth]{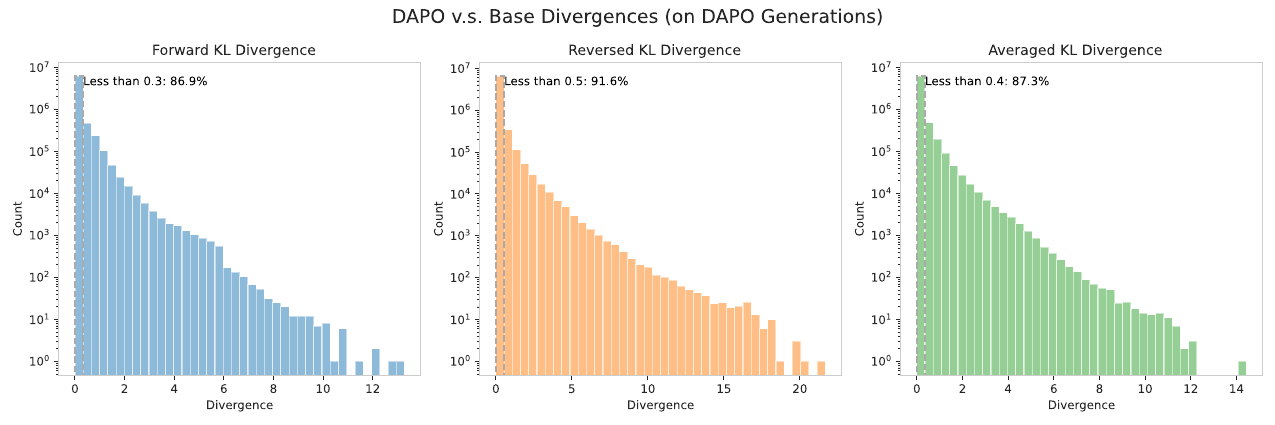}
        \caption{Divergence on DAPO's generations}
    \end{subfigure}
    \hfill
    \begin{subfigure}[c]{0.38\textwidth}
        \centering
        \includegraphics[width=0.9\linewidth]{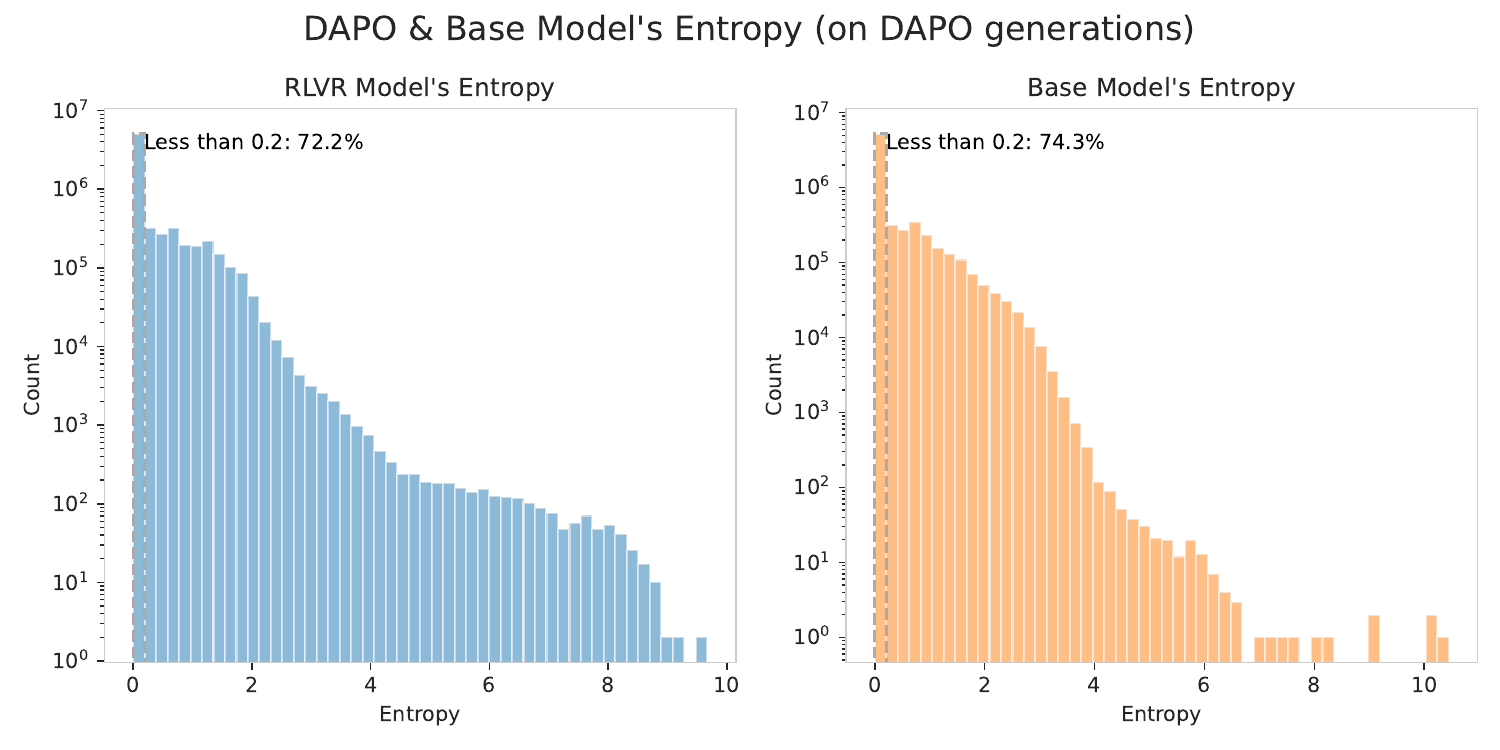}
        \caption{Entropy on DAPO's generations}
    \end{subfigure}
    
    \begin{subfigure}[c]{0.6\textwidth}
        \centering
        \includegraphics[width=0.88\linewidth]{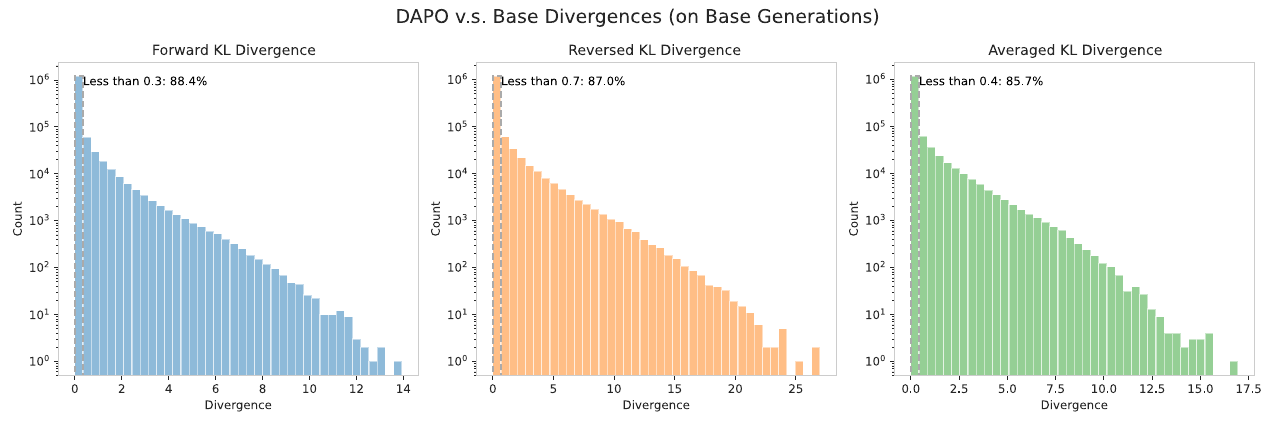}
        \caption{Divergence on base's generations}
    \end{subfigure}
    \hfill
    \begin{subfigure}[c]{0.38\textwidth}
        \centering
        \includegraphics[width=0.9\linewidth]{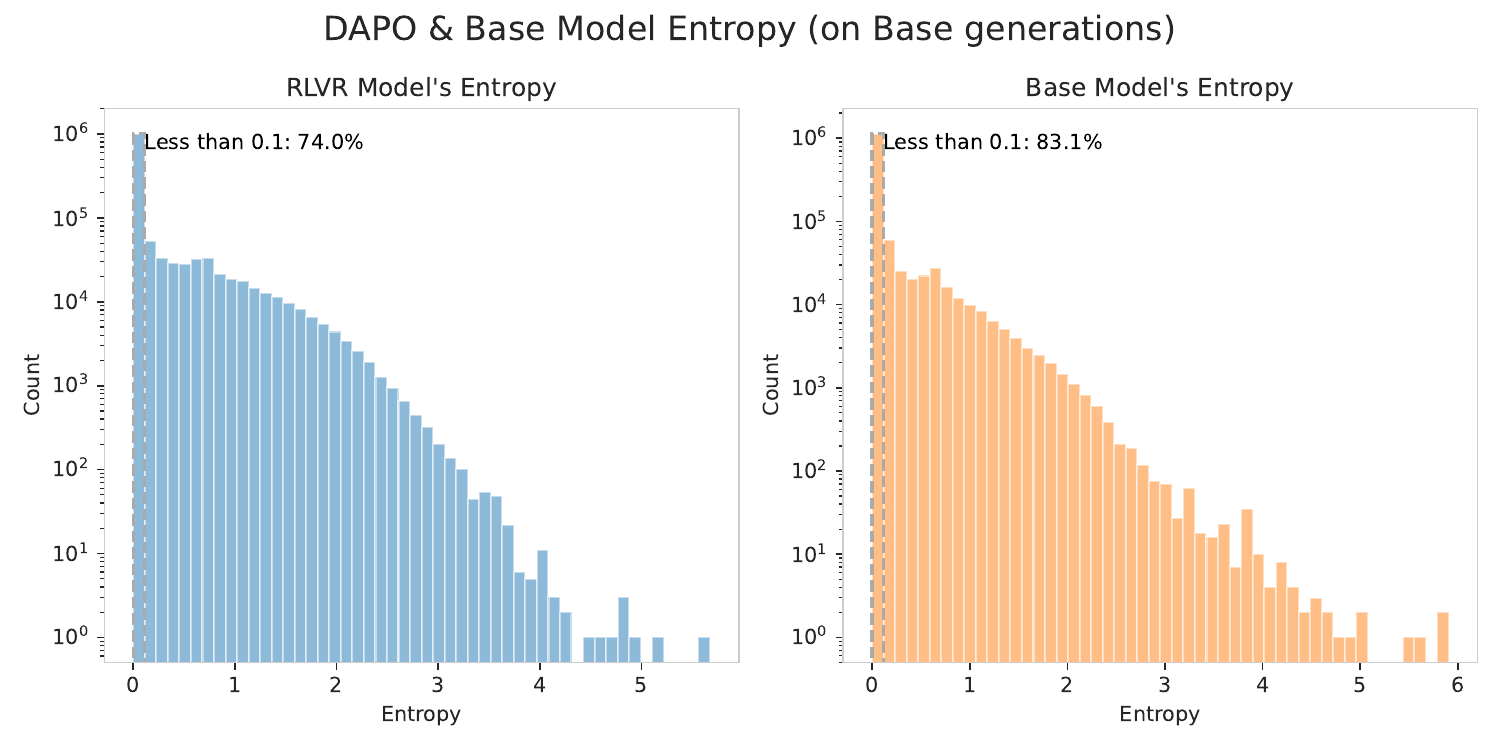}
        \caption{Entropy on base's generations}
    \end{subfigure}

    \caption{Divergence and entropy histograms of DAPO and its corresponding base model measured on DAPO or the base model's generations.}
    \label{fig:hist_dapo_magnitudes}
\end{figure}

\begin{figure}[ht]
    \centering
    \begin{subfigure}[c]{0.6\textwidth}
        \centering
        \includegraphics[width=0.88\linewidth]{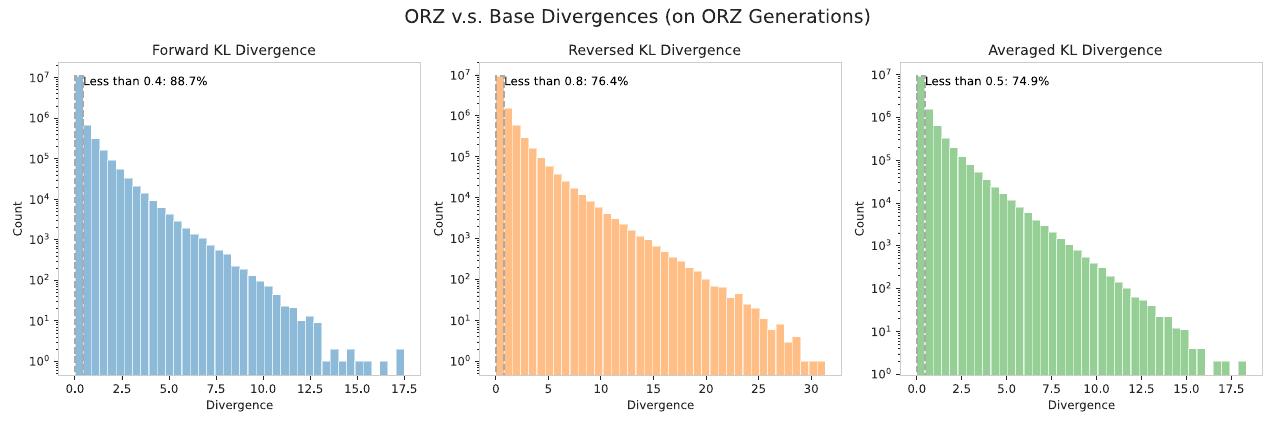}
        \caption{Divergence on ORZ's generations}
    \end{subfigure}
    \hfill
    \begin{subfigure}[c]{0.38\textwidth}
        \centering
        \includegraphics[width=0.9\linewidth]{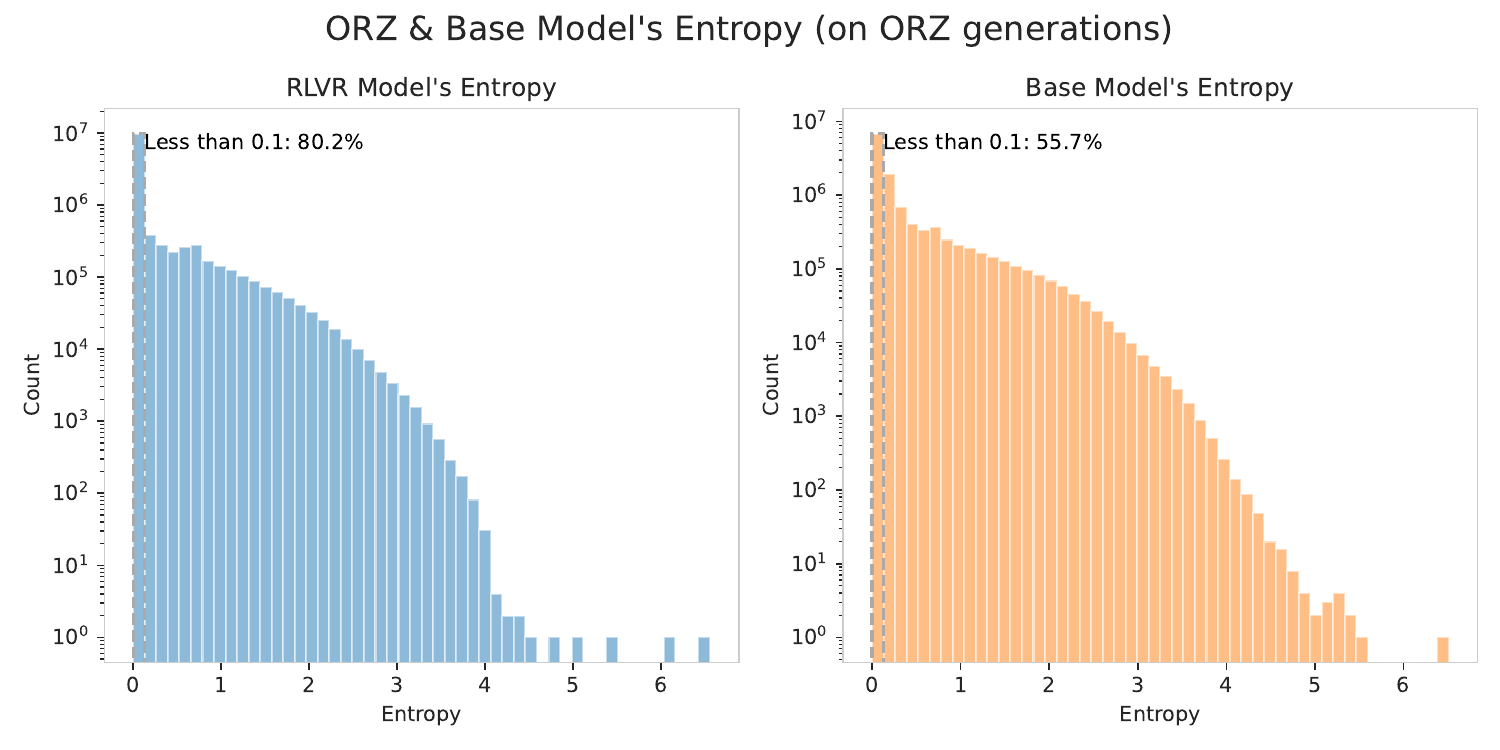}
        \caption{Entropy on ORZ's generations}
    \end{subfigure}
    
    \begin{subfigure}[c]{0.6\textwidth}
        \centering
        \includegraphics[width=0.88\linewidth]{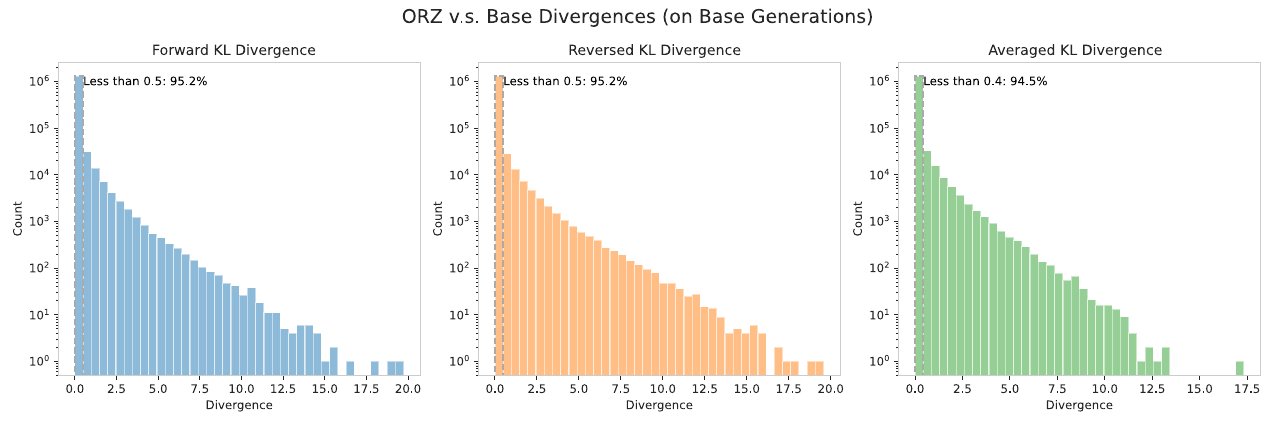}
        \caption{Divergence on base's generations}
    \end{subfigure}
    \hfill
    \begin{subfigure}[c]{0.38\textwidth}
        \centering
        \includegraphics[width=0.9\linewidth]{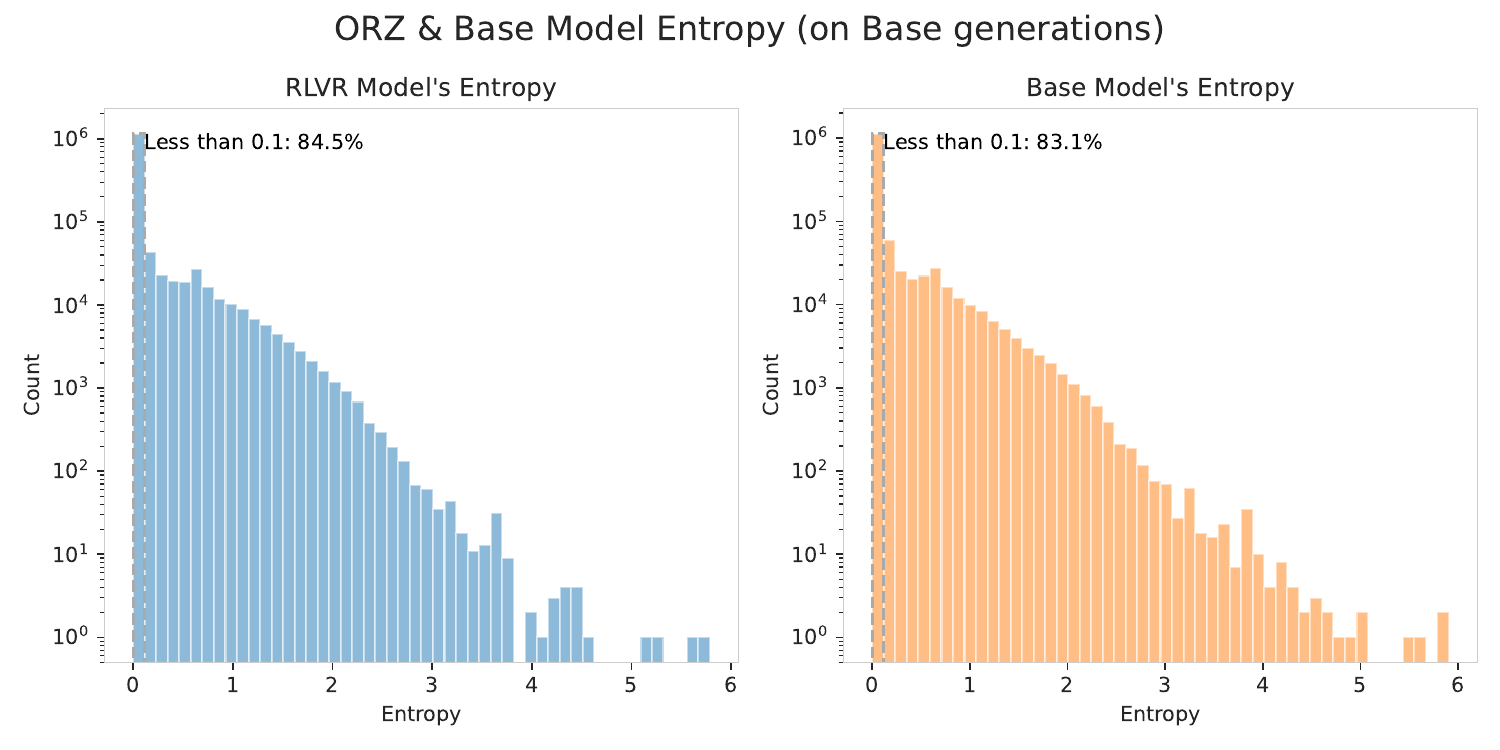}
        \caption{Entropy on base's generations}
    \end{subfigure}

    \caption{Divergence and entropy histograms of ORZ and its corresponding base model measured on ORZ or the base model's generations.}
    \label{fig:hist_orz_magnitudes}
\end{figure}


\end{document}